\documentclass[11pt]{article}

\title{Robust and Computationally Efficient Linear Contextual Bandits under Adversarial Corruption and Heavy-Tailed Noise}
\author{Naoto Tani$^{1,2}$\thanks{u168727h@ecs.osaka-u.ac.jp}, Futoshi Futami$^{1,2,3}$ \\[1ex]
        $^{1}$The University of Osaka \\
        $^{2}$RIKEN AIP \\
        $^{3}$The University of Tokyo}
\date{}       
\usepackage[T1]{fontenc}
\usepackage[utf8]{inputenc}
\usepackage{lmodern}
\usepackage{microtype}
\usepackage[a4paper,margin=25mm]{geometry}

\usepackage{amsmath,amssymb,amsthm,mathtools}

\usepackage{algorithm}
\usepackage{algpseudocode} 

\usepackage{graphicx}
\usepackage{subcaption}
\usepackage{booktabs}
\usepackage{makecell}
\usepackage{subcaption}
\usepackage{minitoc}
\usepackage{adjustbox}

\usepackage{xcolor}
\usepackage{enumitem}

\usepackage[colorlinks=true, citecolor=blue, linkcolor=red, urlcolor=blue]{hyperref}
\usepackage[nameinlink,capitalize,noabbrev]{cleveref}
\usepackage[authoryear]{natbib}   
\theoremstyle{plain}
\newtheorem{theorem}{Theorem}[section]
\newtheorem{lemma}[theorem]{Lemma}
\newtheorem{corollary}[theorem]{Corollary}

\newtheorem{definition}[theorem]{Definition}
\newtheorem{assumption}[theorem]{Assumption}
\newtheorem{remark}[theorem]{Remark}
\newtheorem{property}{Property}

\newcommand{\Indi}{\mbox{\rm{1}}\hspace{-0.25em}\mbox{\rm{l}}}

\DeclareMathOperator*{\argmin}{arg\,min}
\DeclareMathOperator*{\argmax}{arg\,max}     

\begin{document}
\maketitle

\begin{abstract} \label{sec:abstract}
We study linear contextual bandits under adversarial corruption and heavy-tailed noise with finite $(1+\epsilon)$-th moments for some $\epsilon \in (0,1]$.
Existing work that addresses both adversarial corruption and heavy-tailed noise relies on a finite variance (i.e., finite second-moment) assumption and suffers from computational inefficiency.
We propose a computationally efficient algorithm based on online mirror descent that achieves robustness to both adversarial corruption and heavy-tailed noise. While the existing algorithm incurs $\mathcal{O}(t\log T)$ computational cost, our algorithm reduces this to $\mathcal{O}(1)$ per round.
We establish an additive regret bound consisting of a term depending on the $(1+\epsilon)$-moment bound of the noise and a term depending on the total amount of corruption. In particular, when $\epsilon = 1$, our result recovers existing guarantees under finite-variance assumptions. 
When no corruption is present, it matches the best-known rates for linear contextual bandits with heavy-tailed noise. Moreover, the algorithm requires no prior knowledge of the noise moment bound or the total amount of corruption and still guarantees sublinear regret.
\end{abstract}

\section{Introduction}\label{sec:intro}
Linear contextual bandits provide a general framework for sequential
decision-making with side information \citep{lattimore2020bandit}.
Specifically, over a time horizon $T$, in each round $t$, the learner selects an action associated with a context vector $X_t$ and observes a stochastic reward $r_t = \langle \theta^*, X_t \rangle + \eta_t,$ where $\theta^*$ is an unknown parameter and $\eta_t$ is zero-mean noise. The goal of the learner is to select high-reward actions so as to maximize the cumulative reward over $T$. In existing works \citep{abbasi2011improved, agrawal2013thompson}, which enables the use of self-normalized concentration inequalities in the analysis.

However, in practical applications, this reward model can be violated due to two major challenges.
\textbf{First, the stochastic reward may be adversarially corrupted.}
Adversarial corruption arises in many real-world systems, including the presence
of spam or malicious reviews in recommendation systems
\citep{lykouris2018stochastic}.
It can strategically bias parameter estimation, preventing the learner from maximizing cumulative reward.
\textbf{Second, the noise can be heavy-tailed.}
Heavy-tailed noise is commonly observed in domains such as finance \citep{rachev2003handbook}. It can produce extreme reward observations that destabilize parameter estimation.

Since adversarial corruption and heavy-tailed noise can arise simultaneously in practice, algorithms must be robust to both challenges at the same time. In online advertising systems~\citep{li2010contextual}, click fraud introduces adversarial corruption~\citep{lykouris2018stochastic}, while user-interaction signals can also exhibit heavy-tailed behavior~\citep{jebarajakirthy2021mobile}.
In such situation, existing methods that address only one of these challenges \citep{he2022nearly, wang2025heavy} cannot guarantee robustness to the other challenge.

Although \citet{yucorruption} are the first to study contextual bandits under simultaneous adversarial corruption and heavy-tailed noise, their guarantees rely on a finite-variance assumption on the noise (i.e., a bounded second moment).
This requirement is stronger than the bounded $(1+\epsilon)$-moment assumption for some $\epsilon\in (0,1]$ adopted in linear contextual bandits with general heavy-tailed noise \citep{huang2023tackling, wang2025heavy}.
In such cases, their algorithm cannot be directly applied.

In addition to this statistical limitation, the algorithm of \citet{yucorruption} is computationally inefficient. At each round $t$, their estimator is obtained by solving a convex optimization problem over all past observations, resulting in an $\mathcal{O}(t \log T)$ per-round computational cost. 
Such a per-round computational cost that grows with $t$ is undesirable in large-scale or real-time applications.
In contrast, computationally efficient algorithms with $\mathcal{O}(1)$
per-round cost have been developed for handling either adversarial corruption
or heavy-tailed noise~\citep{he2022nearly, wang2025heavy}.

Taken together, existing works do not provide an algorithm that is simultaneously 
robust to adversarial corruption and heavy-tailed noise under bounded $(1+\epsilon)$-moment conditions 
while also achieving per-round $\mathcal{O}(1)$ computational cost. Robustness, in this context, means achieving sublinear regret despite the presence of adversarial corruption or heavy-tailed noise.
In this paper, we study linear contextual bandits under simultaneous adversarial corruption 
and heavy-tailed noise satisfying a bounded $(1+\epsilon)$-moment assumption. 
Under this setting, we address both robustness and computational efficiency.
We summarize our main contributions below.
\begin{itemize}
  \item We propose a new algorithm, corruption-robust heavy-tailed UCB (CR-Hvt-UCB, Algorithm~\ref{alg:CR-HvtUCB}). It achieves robustness to both adversarial corruption and heavy-tailed noise with bounded $(1+\epsilon)$-th moments for any $\epsilon \in (0,1]$, while admitting $\mathcal{O}(1)$ per-round updates. This generalizes the finite-variance setting of existing work while improving its computational efficiency. 
The algorithm performs an online mirror descent (OMD) update~\eqref{def:theta_OMD2} on the Huber-based loss~\eqref{def:huber_loss2}.
By adaptively scaling the loss~\eqref{def:huber_loss2},
we can control the influence of both adversarial corruption and heavy-tailed noise
within the OMD update, thereby achieving both robustness and computational efficiency.

  \item We establish regret guarantees for Algorithm~\ref{alg:CR-HvtUCB} when the noise moment bounds and the total amount of corruption are known. (Theorem~\ref{thm:regret_instance}). The regret bound scales with the square root of the cumulative squared $(1+\epsilon)$-moment bounds of the noise, plus an additive term linear in the total amount of corruption. In particular, when $\epsilon = 1$, our bound recovers the regret guarantee of \citet{yucorruption} for finite-variance noise. In the absence of adversarial corruption, the bound matches the best-known rate for linear contextual bandits with heavy-tailed noise. Importantly, our algorithm works even when the total amount of adversarial corruption and the $(1+\epsilon)$-moment bounds of the noise are unknown, while retaining theoretical regret guarantees (Corollary~\ref{col:regret_c},~\ref{col:regret_nu},~\ref{col:regret_nu_c}).
\end{itemize}

The remainder of the paper is organized as follows.
Section~\ref{sec:related_work} discusses related work and Section~\ref{sec:problem_setting} formalizes the problem setting. Section~\ref{sec:method} presents our algorithm together with its regret guarantees.
Section~\ref{sec:proof_sketch} offers a proof sketch of the main result.
Section~\ref{sec:experiments} reports experimental results, and Section~\ref{sec:conclusion} concludes the paper.
\section{Related Work} \label{sec:related_work}
We position our work relative to prior studies along two key aspects: adversarial corruption and heavy-tailed noise.
In our setting, stochastic rewards are corrupted by an additive adversarial term 
with total magnitude constrained by a budget $C$, and the heavy-tailed noise satisfies a bounded $(1+\epsilon)$-moment condition, i.e.,
$\mathbb{E}\left[|\eta_t|^{1+\epsilon}\right] \le \nu_t^{1+\epsilon}$ for some $\epsilon \in (0,1].$
A detailed and formal problem setting is deferred to Section~\ref{subsec:problem_setting}.
Comparisons with additional existing works are provided in
Appendix~\ref{sec:add_related_work}.

\begin{table}[ht]
\centering
\begin{adjustbox}{width=\textwidth}
{\renewcommand{\arraystretch}{1.25}
\begin{tabular}{c|c|c|c|c}
  \hline
  Paper & C-Robust & HT-Robust & Efficiency & Regret \\ \hline
  
  \citet{abbasi2011improved}  
  & No
  & No
  & $\mathcal{O}(1)$
  & $\tilde{\mathcal{O}}(d\sqrt{T})$
  \\ \hline

  \citet{zhang2025generalized}
  & No
  & No
  & $\mathcal{O}(1)$
  & $\tilde{\mathcal{O}}(d\sqrt{T} )$
  \\ \hline

  \citet{he2022nearly}
  & Yes
  & No
  & $\mathcal{O}(1)$
  & $\tilde{\mathcal{O}}(d\sqrt{T}+dC)$
  \\ \hline

  \citet{wang2025heavy}
  & No
  & Yes
  & $\mathcal{O}(1)$
  & \makecell{ $\tilde{\mathcal{O}}\left(dT^{\frac{1-\epsilon}{2(1+\epsilon)}} \sqrt{\sum_{t=1}^T \nu_t^2} + dT^{\frac{1-\epsilon}{2(1+\epsilon)}}\right) $ \\ $\tilde{\mathcal{O}}(dT^{\frac{1}{1+\epsilon}})$}
  \\ \hline

  \citet{yucorruption}
  & Yes
  & $\epsilon=1$ only
  & $\mathcal{O}(t \log T)$
  & \makecell{  $\tilde{\mathcal{O}}\left(d\sqrt{\sum_{t=1}^T \nu_t^2} + d\cdot 1 \vee C\right) $ \\ $\tilde{\mathcal{O}}( d\sqrt{T}+ d\cdot 1 \vee C)$}
  \\ \hline

  Our work
  & Yes
  & Yes
  & $\mathcal{O}(1)$
  & \makecell{ $\tilde{\mathcal{O}}\left(dT^{\frac{1-\epsilon}{2(1+\epsilon)}} \sqrt{\sum_{t=1}^T \nu_t^2} + dT^{\frac{1-\epsilon}{2(1+\epsilon)}}\cdot 1 \vee C\right) $\\$\tilde{\mathcal{O}}(dT^{\frac{1}{1+\epsilon}} + dT^{\frac{1-\epsilon}{2(1+\epsilon)}}\cdot 1 \vee C)$}
  \\ \hline
\end{tabular}
}
\end{adjustbox}
\caption{We compare our algorithm with existing methods in terms of three key aspects—robustness to adversarial corruption (C-Robust), robustness to heavy-tailed noise (HT-Robust), and computational cost (Efficiency)—and summarize the regret. 
Here, the second row in the Regret column corresponds to the case where
$\nu_t \le \nu$ for all $t$. We define $1\vee C = \max\{1, C\}$.
}
\label{tab:papers1}
\end{table}

\citet{he2022nearly} study linear contextual bandits under adversarial corruption and adopt the same additive corruption model as ours. 
Assuming $R$-sub-Gaussian noise, they establish a nearly optimal regret bound of 
$\tilde{\mathcal{O}}(d\sqrt{T} + dC)$. 
Their analysis relies on self-normalized concentration inequalities for sub-Gaussian noise, and their algorithm knowledge of the corresponding parameter $R$. As bounded $(1+\epsilon)$-moment noise need not be sub-Gaussian, their framework does not directly extend to our setting.
In contrast, we develop a robust algorithm and analysis that remain valid under bounded $(1+\epsilon)$-moment noise. 
Since sub-Gaussian noise implies finite variance (corresponding to $\epsilon = 1$), 
our regret bound recovers the same order $\tilde{\mathcal{O}}(d\sqrt{T} + dC)$ when the moment bound is constant (see Table~\ref{tab:papers1}).

\citet{wang2025heavy} study linear contextual bandits with heavy-tailed noise under the same bounded $(1+\epsilon)$-moment assumption as ours.
Under this assumption, they establish instance-dependent regret bounds of $\tilde{\mathcal{O}}\big(
dT^{\frac{1-\epsilon}{2(1+\epsilon)}} \sqrt{\sum_{t=1}^T \nu_t^2}
+ dT^{\frac{1-\epsilon}{2(1+\epsilon)}}\big).$
However, their algorithm and analysis cannot be directly applied to our setting. In our setting, the quantities to which their concentration inequalities are applied contain an additional corruption component, which invalidates the stochastic structure required by their concentration arguments. Therefore, we develop a new algorithm and a corresponding analysis that are robust to both adversarial corruption and heavy-tailed noise. Moreover,
when the corruption budget satisfies $C=0$, our regret bound recovers
their rate (see Table~\ref{tab:papers1}).

\citet{yucorruption} study generalized linear bandits (GLBs) under simultaneous adversarial corruption and heavy-tailed noise, and adopt the same additive corruption model as ours. 
Under the finite-variance assumption $\mathbb{E}[\eta_t^2] \le \nu_t^2$, they establish regret bounds of $\tilde{\mathcal{O}}\big(
d \sqrt{\sum_{t=1}^T \nu_t^2}
+
d\cdot1 \vee C
\big).$
Their guarantees are restricted to the finite-variance setting.
In contrast, we relax the finite-variance assumption to a bounded $(1+\epsilon)$-moment assumption for some $\epsilon \in (0,1]$ and establish regret guarantees under this weaker assumption. 
Notably, when $\epsilon = 1$, our regret bound recovers the same order as theirs (see Table~\ref{tab:papers1}).

We next discuss prior work from the perspective of computational efficiency. 
Further comparisons are deferred to Appendix~\ref{sec:add_related_work}.

In linear bandits, the OFUL algorithm \citep{abbasi2011improved}
achieves $\mathcal{O}(1)$ per-round computational cost via closed-form updates based on ridge regression.
In linear contextual bandits with adversarial corruption,
\citet{he2022nearly} achieve $\mathcal{O}(1)$ per-round computational cost through similar closed-form updates. 
By contrast, \citet{yucorruption} consider GLBs with both adversarial corruption and heavy-tailed noise, whose algorithm does not admit closed-form updates.
As a result, it requires solving an iterative optimization problem at every round,
leading to $\mathcal{O}(t \log T)$ per-round computational cost.
More broadly, even when closed-form estimators are not available, efficient per-round updates can still be achieved through OMD. \citet{zhang2024online} first demonstrated that an OMD-based estimator
yields $\mathcal{O}(1)$ computational cost per round in logistic bandits, and
\citet{zhang2025generalized} extended this framework to GLBs.
Building on this development, \citet{wang2025heavy} applied an OMD-based
estimator to linear contextual bandits with heavy-tailed noise, showing
that efficient updates remain possible even under heavy-tailed noise.
Our work extends this OMD framework to our setting and shows that efficient updates remain possible even in the presence of both adversarial corruption and heavy-tailed noise.
\section{Problem Setting and Challenges}
\label{sec:problem_setting}
In this section, we introduce the problem setting (Section~\ref{subsec:problem_setting}) and review the limitations of existing work (Section~\ref{subsec:limitations}).

\subsection{Problem Setting} 
\label{subsec:problem_setting}
\textbf{Notation.}
Let $\mathbb{R}^+=[0,\infty)$, and $I_d$ be the $d\times d$ identity matrix. For $n\in \mathbb{N}$, we define $[n] := \{1,2,\dots,n\}.$ For $a,b \in \mathbb{R}$, we write
$a \vee b := \max\{a,b\}.$ For a scalar or a vector $\mathbf{x}\in \mathbb{R}^d$ and a positive semidefinite matrix $A$, we denote the vector's Euclidean norm as $\|\mathbf{x}\|_2$ and the Mahalanobis norm as $\|\mathbf{x}\|_A=\sqrt{\mathbf{x}^\top A\mathbf{x}}$. We denote by $\Indi\{\cdot\}$ the indicator function, that is,
$\Indi_E=1$ if $E$ is true and $0$ otherwise. For two positive functions $f$ and $g$ defined on the positive integers,
we write $f(T)=\mathcal{O}(g(T))$ if there exist constants $c>0$ and
$T_0\in\mathbb{N}$ such that
$f(T)\le c g(T)$ for all $T\ge T_0$. We also use $\tilde{\mathcal{O}}(\cdot)$ to hide polylogarithmic factors.

We define linear contextual bandits with adversarial corruption and heavy-tailed noise. At each round $t\in [T]$, the learner observes a decision set $\mathcal{X}_t\subset\mathbb{R}^d$ and selects an action $X_t\in\mathcal{X}_t$. After choosing the action $X_t$, the environment generates the stochastic reward based on the linear model 
\begin{equation}
\label{equ:stoc_reward}
r^{\prime}_t = \langle X_t, \theta^* \rangle + \eta_t,   
\end{equation} where $\theta^*\in\mathbb{R}^d$ is an unknown parameter vector and $\eta_t$ is a heavy-tailed noise term satisfying the following moment assumption.
\begin{assumption}
\label{ass:noise}
The heavy-tailed noise $\eta_t$ satisfies, for some $\epsilon\in (0,1]$, $\mathbb{E}[\eta_t \mid X_{1:t}, \eta_{1:t-1}, c_{1:t-1}]\\=0$ and $\mathbb{E}[|\eta_t|^{1+\epsilon} \mid X_{1:t}, \eta_{1:t-1}, c_{1:t-1}] \le \nu_t^{1+\epsilon}$, where $\nu_t>0$ is a known or unknown constant to the learner.
\end{assumption}
After generating the stochastic reward, an adversary may corrupt it by adding a corruption term $c_t$. This corruption may depend on the decision set $\mathcal{X}_t$, the action $X_t$, and the stochastic reward $r^{\prime}_t$. Finally, the learner observes the corrupted reward
\begin{equation}
r_t =  r_t^{\prime} +c_t =\langle X_t, \theta^* \rangle + \eta_t + c_t.
\label{eq:reward}
\end{equation}
To measure the total corruption, we define the corruption level $C = \sum_{t=1}^T |c_t|$, which may be known or unknown to the learner.

Following prior work \citep{wang2025heavy,yucorruption}, we make the following assumptions on this setting.
\begin{assumption}
\label{ass:norm}
For all $t$ and all $\mathbf{x} \in \mathcal{X}_t$, we assume $\|\mathbf{x}\|_2 \le L$ for some $L > 0$, and $\theta^* \in \Theta := \{\theta : \|\theta\|_2 \le S\}$ for some $S > 0$.
\end{assumption}

Under this model, the goal of the learner is to minimize the regret
\begin{equation*}
\mathrm{Reg}(T)= \sum_{t=1}^T \max_{\mathbf{x} \in \mathcal{X}_t} \langle \mathbf{x}, \theta^* \rangle - \sum_{t=1}^T \langle X_t, \theta^* \rangle.
\end{equation*}

\subsection{Limitations of Existing Work} \label{subsec:limitations}
To understand the statistical and computational limitations of existing work, 
we examine the method of \citet{yucorruption}.
They employ the Pseudo-Huber loss (Definition~\ref{def:pseudo_huber}), which is a smooth approximation to the Huber loss~\citep{huber1964robust}.
\begin{definition}[Pseudo-Huber loss]
\label{def:pseudo_huber}
Let $\tau > 0$ be a robustification parameter.
For $x \in \mathbb{R}$, the Pseudo-Huber loss $\phi_\tau(x)$ is defined as
\[
\phi_\tau(x)
=
\tau^2 \left( \sqrt{\tau^2 + x^2} - 1 \right).
\]
\end{definition}
Based on this loss, their estimator is constructed at each round $t$ by performing adaptive Huber regression~\citep{sun2020adaptive} using all past observations.
\begin{equation}
\hat{\theta}_t
=
\argmin_{\theta\in\Theta}
\sum_{s=1}^{t}
\phi_{\tau_s}\left(\frac{r_s - \langle X_s , \theta \rangle}{\sigma_s}\right)
+
\frac{\lambda}{2} \|\theta\|_2^2,
\label{def:yu_estimator}
\end{equation}
where $\lambda > 0$ is a regularization parameter,
$\sigma_s >0$ denotes a scale parameter. They propose an OFUL-type algorithm built upon this estimator.

Their robustness to heavy-tailed noise guarantee relies on a finite-variance
assumption on the noise, namely
$\mathbb{E}[\eta_t^2 \mid X_{1:t}, \eta_{1:t-1}, c_{1:t-1}] \le \nu_t^2.$
This is because their algorithm
is built on the algorithmic framework of \citet{li2023variance},
which fundamentally requires bounded second moments.
In particular, the scale parameter $\sigma_t$
explicitly requires an upper bound on the variance.
As a result, it does not naturally extend to our setting
where only a bounded $(1+\epsilon)$-moment assumption holds. 

Their algorithm is computationally expensive.
The estimator~\eqref{def:yu_estimator}
does not admit a closed-form solution
and must be computed via iterative convex optimization.
For estimators that require solving a convex optimization problem at each round $t$,
\citet{wang2025heavy} show that
achieving the best-known regret bound
incurs an $\mathcal{O}(t \log T)$ per-round computational cost.
\section{Method}
\label{sec:method}
In this section, we present our algorithm (CR-Hvt-UCB, Algorithm~\ref{alg:CR-HvtUCB}).
Section~\ref{subsec:our_algorithm} describes its design,
and Section~\ref{subsec:regret} establishes the corresponding regret guarantee.

\subsection{Our algorithm}
\label{subsec:our_algorithm}
Our algorithm achieves robustness
to adversarial corruption and heavy-tailed noise with bounded $(1+\epsilon)$ moments
while achieving $\mathcal{O}(1)$ per-round updates. It is built upon a Huber-based robust estimator
combined with an update rule based on OMD.

Following prior Huber-based approaches for linear contextual bandits under bounded $(1+\epsilon)$-moment assumptions
\citep{huang2023tackling, wang2025heavy}, 
we adopt the Huber loss~\citep{huber1964robust}.
\begin{definition}[Huber loss]
Let $\tau > 0$ be a threshold parameter.  
For $x \in \mathbb{R}$, the Huber loss is defined as
\begin{equation}
f_{\tau}(x) :=
\begin{cases}
  \frac{x^2}{2}, & |x| \le \tau, \\
  \tau |x| - \frac{\tau^2}{2}, & |x| > \tau.
\end{cases}
\label{def:huber_loss1}
\end{equation}
\end{definition}
The Huber loss coincides with the squared loss when $|x|\le \tau$
and transitions to the absolute loss when $|x|>\tau$.
Consequently, its gradient is bounded by $\tau$, which prevents
extreme rewards from dominating the parameter estimation.

We employ the following Huber loss applied to the normalized prediction error $z_t(\theta) := (r_t - \theta^\top X_t)/\sigma_t,$ where the scale parameter $\sigma_t > 0$
will be specified later in~\eqref{eq:sigma_def}.
\begin{equation}
\ell_t(\theta) :=
\begin{cases}
    \frac{1}{2} z_t(\theta)^2, & |z_t(\theta)| \le \tau_t, \\
    \tau_t |z_t(\theta)| - \frac{\tau_t^2}{2}, & |z_t(\theta)| > \tau_t.
\end{cases}
\label{def:huber_loss2}
\end{equation}
The threshold parameter $\tau_t$ will be specified later in~\eqref{eq:tau_def}.

Existing Huber-based approaches
\citep{huang2023tackling, wang2025heavy}
design the parameters $\sigma_t$ and $\tau_t$
to ensure robustness to heavy-tailed noise. To achieve robustness to both adversarial corruption and heavy-tailed noise,
we redesign $\sigma_t$ and $\tau_t$ as follows:
\begin{align}
\sigma_t 
&:= \max\Biggl\{
\nu_t,  
 \sigma_{\min},  
 \sqrt{\frac{2\beta_{t-1}}
            {\tau_0 \sqrt{\alpha}
             t^{\frac{1-\epsilon}{2(1+\epsilon)}}}}
 \|X_t\|_{V_{t-1}^{-1}},  
 \sqrt{C}\kappa^{-1/4}
   \|X_t\|_{V_{t-1}^{-1}}^{1/2}
\Biggr\},
\label{eq:sigma_def}
\\
\tau_t 
&:= \tau_0 \frac{\sqrt{1+w_t^2}}{w_t}
      t^{\frac{1-\epsilon}{2(1+\epsilon)}}.
\label{eq:tau_def}
\end{align}
The matrix $V_t$ denotes the data-dependent positive definite matrix specified in~\eqref{def:regular_func}.
Here, $\sigma_{\min}>0$ is a small constant specified in
Lemma~\ref{lem:conf_interval},
and $\beta_t$ denotes the confidence radius
defined in the same lemma. The parameter $\alpha>0$ will be specified following~\eqref{def:regular_func}. We further define the problem-dependent constant $\kappa = d\log\!\left(1+\frac{L^2 T}{\sigma_{\min}^2\lambda \alpha d}\right).$
The auxiliary quantities $\tau_0$ and $w_t$ are given by
\begin{align}
\tau_0 
= 
\frac{\sqrt{2\kappa}(\log 3T)^{\frac{1-\epsilon}{2(1+\epsilon)}}}
     {\bigl(\log \frac{2T^2}{\delta}\bigr)^{\frac{1}{1+\epsilon}}}, 
w_t 
= 
\frac{1}{\sqrt{\alpha}}
\left\|
    \frac{X_t}{\sigma_t}
\right\|_{V^{-1}_{t-1}}.
\label{eq:auxi_def}
\end{align}
An intuitive explanation for why this design yields robustness to corruption is as follows.
The corruption-dependent term in $\sigma_t$ inflates the normalization scale
when large corruption is present. Consequently, the normalized prediction error
$z_t(\theta) = (r_t - \theta^\top X_t)/\sigma_t$
shrinks in magnitude, leading to a smaller gradient of the loss~\eqref{def:huber_loss2}. 
This smaller gradient suppresses the influence of corruption
on the parameter estimation. With this design, the loss~\eqref{def:huber_loss2} is robust to both adversarial corruption and heavy-tailed noise.

\begin{remark}
Our definition of $\sigma_t$ can be viewed as a direct generalization of the scheme of \citet{wang2025heavy}.
When $C=0$, the corruption-dependent component in \eqref{eq:sigma_def} vanishes,
and $\sigma_t$ reduces exactly to the form used by \citet{wang2025heavy}.
\end{remark}

\begin{remark}
\label{remark:comparison_algo_he_yu}
Since $\kappa = \tilde{\mathcal{O}}(d)$, the corruption-dependent component of $\sigma_t$ in~\eqref{eq:sigma_def} scales as $\tilde{\mathcal{O}}(
\sqrt{C} d^{-1/4}\|X_t\|_{V_{t-1}^{-1}}^{1/2}).$
This matches the order of the corruption-dependent component in the scale parameter $\sigma_t$ of~\citet{yucorruption} defined in~\eqref{def:yu_estimator}.
Moreover, the corruption-dependent weight
in the weighted squared loss of \citet{he2022nearly}
can be expressed in terms of $z_t(\theta)$,
under which it matches the order of the corruption-dependent component of $\sigma_t$ in~\eqref{eq:sigma_def}.
\end{remark}

While the loss~\eqref{def:huber_loss2} ensures robustness to both adversarial corruption and heavy-tailed noise,
directly minimizing it via convex optimization at each round
would incur the same computational cost as in \citet{yucorruption}.
To overcome this issue,
we build on the estimator proposed by \citet{wang2025heavy},
which is based on the online mirror descent (OMD) update~\citep{orabona2019modern}.
Given the current estimate $\hat{\theta}_t\in\Theta$, the OMD update takes the form
\begin{equation}
\hat{\theta}_{t+1}
=\argmin_{\theta\in\Theta}
\Big\langle \theta, \nabla \ell_t(\hat{\theta}_t)\Big\rangle
+ \mathcal{D}_{\psi_t}(\theta, \hat{\theta}_t),
\label{def:theta_OMD1}
\end{equation}
where $\mathcal{D}_{\psi_t}(\theta,\theta') 
:= \psi_t(\theta) - \psi_t(\theta^\prime) 
- \langle \nabla \psi_t(\theta^\prime), \theta-\theta^\prime \rangle$
denotes the Bregman divergence induced by the regularizer $\psi_t$.
We choose a quadratic regularizer
\begin{equation}
\psi_t(\theta)
= \frac{1}{2} \|\theta\|_{V_t}^2,
\qquad
V_t := V_{t-1} + \frac{1}{\alpha}\frac{X_tX_t^\top}{\sigma_t^2},
\label{def:regular_func}
\end{equation}
with $V_0=\lambda I_d$, where $\lambda>0$ is a regularization parameter and $\alpha>0$ is the step-size
parameter of the OMD update. The positive definite matrix $V_t$ modulates the contribution of each round 
via the weight $1/\sigma_t^2$, 
ensuring that observations more strongly influenced by adversarial corruption or heavy-tailed noise exert a smaller impact on the update.
This choice of regularizer aligns the update geometry with the elliptical norm used for constructing confidence sets in Lemma~\ref{lem:conf_interval}.
The OMD update~\eqref{def:theta_OMD1} admits an equivalent two-step representation:
\begin{equation}
\label{def:theta_OMD2}
\begin{aligned}
\tilde{\theta}_{t+1}
&= \hat{\theta}_t - V_t^{-1} \nabla \ell_t(\hat{\theta}_t), \\
\hat{\theta}_{t+1}
&= \argmin_{\theta\in\Theta}\|\theta - \tilde{\theta}_{t+1}\|_{V_t}.
\end{aligned}
\end{equation}
Focusing on the dependence on $t$ and $T$,
our update~\eqref{def:theta_OMD2} admits a per-round computational cost of $\mathcal{O}(1)$,
whereas the estimator of \citet{yucorruption} in~\eqref{def:yu_estimator} requires $\mathcal{O}(t \log T)$ per round. Our update only requires maintaining
the inverse matrix via the Sherman--Morrison--Woodbury formula~\citep{horn2012matrix}
and performing a projection,
which can be computed in $\mathcal{O}(d^2)$ and $\mathcal{O}(d^3)$ time, respectively.

By combining the loss~\eqref{def:huber_loss2} using two parameters \eqref{eq:sigma_def} and \eqref{eq:tau_def}
with the OMD-based update~\eqref{def:theta_OMD2},
we obtain the estimator that is robust to both adversarial corruption and heavy-tailed noise, while maintaining $\mathcal{O}(1)$ per-round computational cost.
With this estimator~\eqref{def:theta_OMD2} in place, we now specify the arm-selection rule.

We adopt a standard UCB-based arm-selection rule~\citep{auer2002finite, abbasi2011improved}.
\begin{equation}
    X_t = \argmax_{\mathbf{x} \in \mathcal{X}_t}
    \left\{
        \langle \mathbf{x}, \hat{\theta}_t \rangle
        + \beta_{t-1}\|\mathbf{x}\|_{V_{t-1}^{-1}}
    \right\}.
\label{eq:arm_selection}
\end{equation}
Here, when the corruption level $C$ and the $(1+\epsilon)$-moment bound $\nu_t$ are known,
the confidence radius $\beta_t$ is specified in the following lemma.
\begin{lemma}
\label{lem:conf_interval}
Let $\sigma_t$ and $\tau_t$ be defined as in~\eqref{eq:sigma_def} and~\eqref{eq:tau_def},
where $\sigma_{\min} = 1/\sqrt{T}$ and $\kappa=d\log\bigl(1+\frac{L^2T}{4\sigma_{\mathrm{min}}^{2}\lambda d}\bigr)$.
Set the auxiliary parameters $\tau_0$ and $w_t$ as in~\eqref{eq:auxi_def}. Let the step size $\alpha=8$.
For any $\delta\in (0,\frac{1}{4})$, with probability at least $1-4\delta$, for any
$t\in[T]$, 
$\|\hat{\theta}_{t+1}-\theta^*\|_{V_t} \le \beta_t$ holds, where
\begin{equation}
\beta_{t}= 409\log\left(\frac{2T^2}{\delta}\right)\tau_{0}t^{\frac{1-\epsilon}{2(1+\epsilon)}}+\sqrt{\lambda (2+4S^2)}.
\label{def:beta}
\end{equation}
\end{lemma}
The overall procedure is summarized in Algorithm~\ref{alg:CR-HvtUCB}.  

\begin{algorithm}[t]
\caption{CR-Hvt-UCB}
\label{alg:CR-HvtUCB}
\begin{algorithmic}[1]
\State \textbf{Input:} time horizon $T$, corruption level $C$, moment parameter $\epsilon$, norm bounds $S,L$, confidence $\delta$, regularizer $\lambda$, algorithm parameters $\sigma_{\min}, \alpha$.
\State Set $\kappa = d \log\left(1 + \frac{L^2 T}{\sigma_{\min}^2 \lambda \alpha d}\right)$, $\tau_0$ as \eqref{eq:auxi_def}, $V_0 = \lambda I_d$, $\hat{\theta}_1 =  \mathbf{0}$, and compute $\beta_0$ by \eqref{def:beta}.
\For{$t=1,2,\dots,T$}
\State $X_t = \argmax_{\mathbf{x} \in \mathcal{X}_t} \langle \mathbf{x}, \hat{\theta}_t \rangle + \beta_{t-1} \|\mathbf{x}\|_{V_{t-1}^{-1}}$ \qquad\text{(ties broken arbitrarily.)}
\State Observe $r_t$ and $\nu_t$
\State Set $\sigma_t$ as \eqref{eq:sigma_def}
\State $w_t = \frac{1}{\sqrt{\alpha}} \left\| \frac{X_t}{\sigma_t} \right\|_{V_{t-1}^{-1}}$ 
\State $\tau_t = \tau_0  \frac{\sqrt{1 + w_t^2}}{w_t} t^{\frac{1 - \epsilon}{2(1 + \epsilon)}}$
\State $V_t = V_{t-1} + \alpha^{-1} \sigma_t^{-2} X_t X_t^\top$
\State Update $\hat{\theta}_{t+1}$ via \eqref{def:theta_OMD2}, and compute $\beta_t$ by \eqref{def:beta}
\EndFor
\end{algorithmic}
\end{algorithm}

\subsection{Regret guarantee}
\label{subsec:regret}
In this section, we now present the regret guarantee for Algorithm~\ref{alg:CR-HvtUCB}.
We first establish the result under the assumption that
the corruption level $C$ and $(1+\epsilon)$-moment bounds $\nu_t$
are known to the learner.
We then extend the analysis to the case where
one or both of these quantities are unknown.

When both $C$ and $\nu_t$ are known to the learner, our algorithm achieves the following instance-dependent regret bound.
\begin{theorem}
\label{thm:regret_instance}
Let $\sigma_{t}$, $\tau_t$, $\tau_0$, $w_t$, and $\alpha$ be set according to Lemma~\ref{lem:conf_interval}. Choose $\lambda = d$, $\sigma_{\mathrm{min}}=\frac{1}{\sqrt{T}}$, and $\delta = \frac{1}{8T}$.  
Then, with probability at least $1-1/T$, the regret of Algorithm~\ref{alg:CR-HvtUCB} satisfies
\begin{align*}
\mathrm{Reg}(T) = \tilde{\mathcal{O}}\left(dT^{\frac{1-\epsilon}{2(1+\epsilon)}} \sqrt{\sum_{t=1}^T \nu_t^2} + dT^{\frac{1-\epsilon}{2(1+\epsilon)}} \cdot 1\vee C\right).
\end{align*}
\end{theorem}
A complete proof is provided in Appendix~\ref{subsec:ins_reg}.
\begin{remark}
\label{remark:comparison_regret_instance_yu_huang_wang}
Our result can be viewed as a strict generalization of existing bounds.
When $\epsilon = 1$, our regret bound coincides with the bound
established by \citet{yucorruption} for finite-variance noise. Moreover, in the absence of adversarial corruption ($C = 0$), our bound recovers
the regret order for linear contextual bandits with heavy-tailed noise \citep{huang2023tackling, wang2025heavy}. 
\end{remark}

We now consider the setting where $C$ and/or $\nu_t$
are unknown to the learner.
This includes the cases where only $C$ is unknown,
only $\nu_t$ is unknown,
or both are unknown.
In each case, the unknown quantities
are replaced by available upper bounds
$(\bar{C}, \nu)$
in the definition of $\sigma_t$,
where $\bar{C}$ and $\nu$ is a positive constant such that $\bar{C} \ge C$ and $\mathbb{E}[|\eta_t|^{1+\epsilon} \mid X_{1:t}, \eta_{1:t-1}, c_{1:t-1}]
\le \nu^{1+\epsilon}.$
Under these substitutions,
Lemma~\ref{lem:conf_interval} continues to hold.

When $C$ is unknown, the regret is bounded as follows.
\begin{corollary}
\label{col:regret_c}
Under the parameter choices of Theorem~\ref{thm:regret_instance}, replace $C$ by $\bar{C}$ in the definition of $\sigma_t$.  
Then, with probability at least $1-1/T$,
\begin{align*}
\mathrm{Reg}(T) = \tilde{\mathcal{O}}\left(dT^{\frac{1-\epsilon}{2(1+\epsilon)}} \sqrt{\sum_{t=1}^T \nu_t^2}   +   dT^{\frac{1-\epsilon}{2(1+\epsilon)}} \cdot 1\vee\bar{C}\right).
\end{align*}
\end{corollary}
A complete proof is provided in Appendix~\ref{subsec:unknown_c_reg}.

When $\nu_t$ is unknown, the regret satisfies the following upper bound.
\begin{corollary}
\label{col:regret_nu}
Under the parameter choices of Theorem~\ref{thm:regret_instance}, replace $\nu_t$ with $\nu$ in the definition of $\sigma_t$.  
Then, with probability at least $1-1/T$,
\begin{equation*}
\mathrm{Reg}(T)= \tilde{\mathcal{O}}\left(dT^{\frac{1}{1+\epsilon}}   +   dT^{\frac{1-\epsilon}{2(1+\epsilon)}} \cdot 1\vee C\right).
\end{equation*}
\end{corollary}
A complete proof is provided in Appendix~\ref{subsec:unknown_nu_reg}.
\begin{remark}
\label{remark:comparison_he}
If $C=\mathcal{O}(\sqrt{T})$, the regret reduces to 
$\tilde{\mathcal{O}}(dT^{\frac{1}{1+\epsilon}})$,
which coincides with the uncorrupted rate ($C=0$).
This bound does not match the lower bound of order $\tilde{\Omega}\left(d^{\frac{2\epsilon}{1+\epsilon}} T^{\frac{1}{1+\epsilon}}\right)$
under bounded $(1+\epsilon)$-moment assumptions \citep{tajdini2025improved}.
Nevertheless, it matches the best known regret order achieved by algorithms for linear contextual bandits with heavy-tailed noise.
\end{remark}

When both $C$ and $\nu_t$ are unknown, we obtain the following regret guarantee.
\begin{corollary}
\label{col:regret_nu_c}
Under the parameter choices of Theorem~\ref{thm:regret_instance}, replace $(C, \nu_t)$ by $(\bar{C},\nu)$ in the definition of $\sigma_t$.  
Then, with probability at least $1-1/T$,
\begin{equation*}
\mathrm{Reg}(T)= \tilde{\mathcal{O}}\left(dT^{\frac{1}{1+\epsilon}}   +   dT^{\frac{1-\epsilon}{2(1+\epsilon)}} \cdot 1\vee\bar{C}\right).
\end{equation*}
\end{corollary}
\begin{proof}
The claim follows by combining the proofs of Corollary~\ref{col:regret_c} and Corollary~\ref{col:regret_nu}.
\end{proof}
\section{Proof Sketch}
\label{sec:proof_sketch}
In this section, we focus on the technical ideas underlying the proof of Lemma~\ref{lem:conf_interval}. 
Lemma~\ref{lem:conf_interval} is the key ingredient for establishing our regret bound.
Once this lemma is obtained, the regret guarantee follows from standard UCB arguments.

Lemma~\ref{lem:conf_interval} is established by adapting the OMD-based error decomposition introduced in Lemma~3 of \citet{wang2025heavy}. However, in our setting, adversarial corruption—absent in \citet{wang2025heavy}—enters the gradient of the loss~\eqref{def:huber_loss2} through the input $z_t(\theta)$. The key idea is to decompose the gradient appearing in the OMD analysis into components that are independent of the corruption and those that depend on it.
This separation allows us to explicitly isolate and control the contribution of adversarial corruption in the estimation error. 

In our setting, the estimation error can be decomposed as follows (formal statement in Lemma~\ref{lem:key_decomp}):
\begin{align}
&\|\hat{\theta}_{t+1}-\theta^*\|_{V_t}^2 \le
\underbrace{4\lambda S^2}_{\text{regularization term}}
+ \underbrace{\sum_{s=1}^t
\bigl\|\nabla \ell_s(\hat{\theta}_s)\bigr\|_{V_s^{-1}}^2}_{\text{stability term}}
\notag \\
&\quad + \underbrace{2 \sum_{s=1}^t
\left\langle
\nabla\tilde{\ell}_s(\hat{\theta}_s)
- \nabla\ell_s(\hat{\theta}_s),
\hat{\theta}_s-\theta^*
\right\rangle}_{\text{generalization-gap term}}
+ \underbrace{\left(\frac{1}{\alpha}-1\right)
\sum_{s=1}^t
\bigl\|\hat{\theta}_s-\theta^*\bigr\|_{X_sX_s^\top/\sigma_s^2}}_{\text{negative term}}.
\label{equ:lem3_informal}
\end{align}
Here, $\ell_s(\theta)$ is the loss in \eqref{def:huber_loss2},
while $\tilde{\ell}_s(\theta)$ denotes the
noise- and corruption-free Huber loss~\eqref{def:huber_loss1},
evaluated at
$\tilde z_s(\theta)
= \langle X_s, \theta^* - \theta \rangle / \sigma_s$.

From this decomposition, the corruption appears in the stability and generalization-gap terms because of $\nabla \ell_t(\hat{\theta})$. Therefore, by isolating the corruption
and explicitly controlling them,
the analysis of \citet{wang2025heavy} can be extended to our setting.
We now explain how each term is decomposed and controlled.
\paragraph{Stability Term analysis.}
To isolate the contribution of adversarial corruption,
we first decompose the squared gradient norm at each round $s$ as follows:
\begin{align*}
\bigl\|\nabla \ell_s(\hat{\theta}_s)\bigr\|_{V_s^{-1}}^2
&\le
3 \underbrace{
\left(
   \min\left\{
      \left|\frac{\eta_s}{\sigma_s}\right|,
      \tau_s
   \right\}
\right)^{2}
   \left\|
      \frac{X_s}{\sigma_s}
   \right\|_{V_s^{-1}}^{2}
}_{\text{stochastic term}} + 3 \underbrace{
\left(
   \frac{\langle X_s, \theta^* - \hat{\theta}_s\rangle}{\sigma_s}
\right)^{2}
   \left\|
      \frac{X_s}{\sigma_s}
   \right\|_{V_s^{-1}}^{2}
}_{\text{deterministic term}} \\
&\quad
+ 3 \underbrace{
\left(
   \frac{c_s}{\sigma_s}
\right)^{2}
   \left\|
      \frac{X_s}{\sigma_s}
   \right\|_{V_s^{-1}}^{2}
}_{\text{corruption term}}.
\end{align*}
The sum of the first two terms over $s=1,\ldots,t$ coincides with the corresponding terms in \citet{wang2025heavy}, and is therefore bounded by $\tilde{\mathcal O}(dt^{\frac{1-\epsilon}{1+\epsilon}})$. The sum of the third term does not appear in \citet{wang2025heavy}
and captures the effect of adversarial corruption. It is uniformly bounded by $\tilde{\mathcal O}(d)$, using $\sigma_s \ge \sqrt{C}\kappa^{-1/4}\|X_s\|_{V_{s-1}}^{1/2}$ together with $\sum_{s=1}^t c_s^2 \le C^2$.
The stability term is therefore dominated by the first two terms,
and retains the same order
$\tilde{\mathcal O}(dt^{\frac{1-\epsilon}{1+\epsilon}})$
as in \citet{wang2025heavy}. The complete proof of this analysis is given in Appendix~\ref{sec:stability_proof}.
\paragraph{Generalization Gap Term analysis.}
Following \citet{wang2025heavy}, we first decompose the generalization-gap term into the Huber loss term and the self-regularization term at each round $s$ as follows:
\begin{align*}
\left\langle \nabla \tilde{\ell}_s(\hat{\theta}_s) - \nabla \ell_s(\hat{\theta}_s), \hat{\theta}_s - \theta^* \right\rangle 
&=
\underbrace{
\left\langle 
\nabla \tilde{\ell}_s(\hat{\theta}_s) + \nabla \ell_s(\theta^*) - \nabla \ell_s(\hat{\theta}_s), \hat{\theta}_s - \theta^*
\right\rangle 
}_{\text{Huber loss term}} \\
&\quad + 
\underbrace{
 \left\langle 
-\nabla \ell_s(\theta^*), \hat{\theta}_s - \theta^*
\right\rangle  
}_{\text{Self-regularization term}}.
\end{align*}
Since the corruption appears in both the Huber loss and self-regularization terms through $\nabla \ell_s(\theta^*)$,
we furthuer decompose them into corruption-independent and corruption-dependent components,
and control each part separately.
As in \citet{wang2025heavy}, the Huber loss term is supported only on the event $\bigl\{|z_s(\theta^*)|>\tfrac{\tau_s}{2}\bigr\}.$
Since $z_s(\theta^*) = (\eta_s + c_s)/\sigma_s$ contains the corruption term additively,
we apply a union bound to decompose the event
into corruption-independent and corruption-dependent parts.
\begin{equation*}
\Indi\Bigl\{|z_s(\theta^*)|>\tfrac{\tau_s}{2}\Bigr\}
\le
\Indi\Bigl\{\Bigl|\tfrac{\eta_s}{\sigma_s}\Bigr|>\tfrac{\tau_s}{4}\Bigr\}
+
\Indi\Bigl\{\Bigl|\tfrac{c_s}{\sigma_s}\Bigr|>\tfrac{\tau_s}{4}\Bigr\}.
\end{equation*}
The sum of corruption-independent part over $s=1,\ldots,t$ admits the same bound as in
\citet{wang2025heavy}, namely $\tilde{\mathcal O}(\sqrt{d}t^{\frac{1-\varepsilon}{2(1+\varepsilon)}}\beta_t)$. The sum of corruption-dependent part can be bounded $\tilde{\mathcal O}(\sqrt{d}t^{\frac{1-\varepsilon}{2(1+\varepsilon)}}\beta_t)$ by the definition of $\tau_s$
together with $\sigma_s \ge \sqrt{C}\kappa^{-1/4}\|X_s\|_{V_{s-1}}^{1/2}$. 
Since the corruption-dependent part is of the same order as the
corruption-independent part,
the Huber loss term admits the same order
$\tilde{\mathcal O}(\sqrt{d}t^{\frac{1-\varepsilon}{2(1+\varepsilon)}}\beta_t)$
as in \citet{wang2025heavy}.
For the self-regularization term, by the chain rule, we have $\nabla \ell_s(\theta)=-h_s\left(z_s(\theta)\right)\frac{X_s}{\sigma_s},$
where $h_s:\mathbb R\to\mathbb R$ denotes the derivative of the Huber loss.
Since $h_s$ is $1$-Lipschitz, \(
\bigl|h_s((\eta_s+c_s)/\sigma_s)-h_s(\eta_s/\sigma_s)\bigr|
\le |c_s|/\sigma_s.
\) which allows us to separate the corruption-dependent component as follows:
\begin{align*}
\left\langle 
-\nabla \ell_s(\theta^*), \hat{\theta}_s - \theta^*
\right\rangle
&=\left\langle 
h_s\left(\frac{\eta_s+c_s}{\sigma_s}\right), \hat{\theta}_s - \theta^*
\right\rangle \\
&\le
h_s\!\left(\frac{\eta_s}{\sigma_s}\right)
\left\langle 
\frac{X_s}{\sigma_s}, \hat{\theta}_s - \theta^*
\right\rangle  
+
\frac{|c_s|}{\sigma_s}
\left\langle 
\frac{X_s}{\sigma_s}, \hat{\theta}_s - \theta^*
\right\rangle .
\end{align*}
The sum of the corruption-independent part over $s=1,\dots,t$
is bounded as in \citet{wang2025heavy}, namely
$\tilde{\mathcal{O}}(dt^{\frac{1-\epsilon}{1+\epsilon}})$.
The sum of the corruption-dependent part is of order
$\tilde{\mathcal O}(\sqrt{d}\beta_t)$,
which follows from $\sigma_s \ge \sqrt{C}\kappa^{-1/4}\|X_s\|_{V_{s-1}}^{1/2}$.
Since this part is dominated by the Huber loss term,
the overall bound coincides with that of \citet{wang2025heavy}. The complete proof of this analysis is given in Appendeix~\ref{sec:gene_gap_proof}.

Therefore, despite the presence of adversarial corruption,
both the stability and generalization-gap terms retain the same order
as in \citet{wang2025heavy}.
Consequently, the overall decomposition preserves the same order,
and a confidence radius of
$\tilde{\mathcal{O}}(\sqrt{d}t^{\frac{1-\epsilon}{2(1+\epsilon)}})$
suffices to establish Lemma~\ref{lem:conf_interval}.
The complete proof of Lemma~\ref{lem:conf_interval} is given in
Appendix~\ref{sec:conf_interval_proof}.
\section{Experiments}
\label{sec:experiments}

\begin{figure*}[t] \centering \begin{subfigure}{0.33\textwidth} \centering \includegraphics[width=\linewidth]{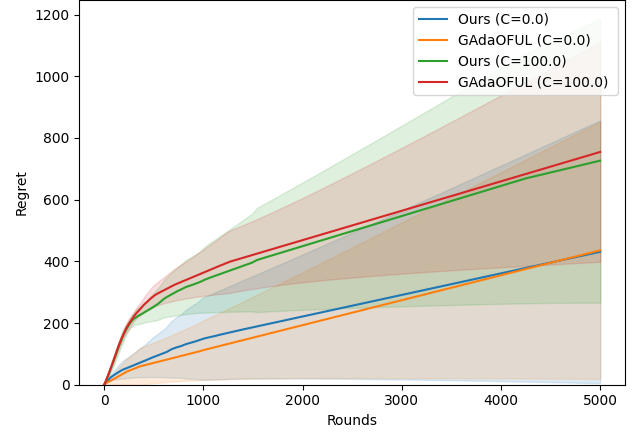} \caption{Regret ($\epsilon=1$)} \label{fig:tdist_regret} \end{subfigure}\hfill \begin{subfigure}{0.33\textwidth} \centering \includegraphics[width=\linewidth]{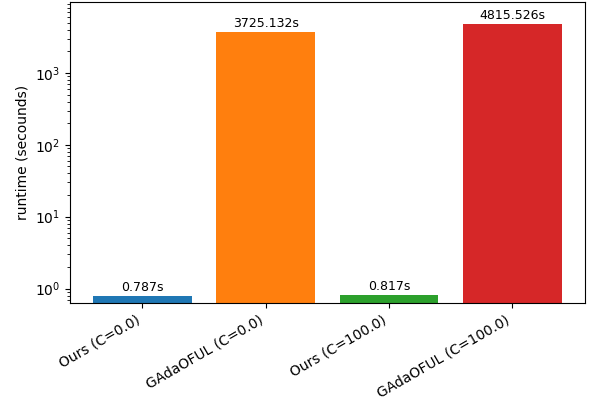} \caption{Runtime ($\epsilon=1$)} \label{fig:tdist_runtime} \end{subfigure}\hfill \begin{subfigure}{0.33\textwidth} \centering \includegraphics[width=\linewidth]{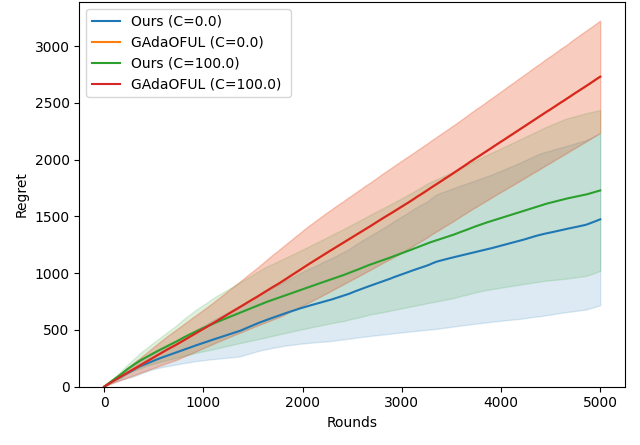} \caption{Regret ($\epsilon<1$)} \label{fig:pareto_regret} \end{subfigure} \caption{Comparison ours and GAdaOFUL} \label{fig:exp} \end{figure*}

In this section, we numerically evaluate the effectiveness of our algorithm (Algorithm~\ref{alg:CR-HvtUCB}) by comparing it with the algorithm of \citet{yucorruption}.
We consider a linear contextual bandit with $d=10$. We sample $\theta^*$ uniformly from the unit sphere in $\mathbb{R}^d$.
At each round $t$, the decision set $\mathcal{X}_t$ consists of $20$
independent random unit vectors in $\mathbb{R}^d$ generated in the same manner.
Thus, in accordance with Assumption~\ref{ass:norm}, we set $L=1$ and $S=1$.
The stochastic reward $r_t^{\prime}$ is generated as in~\eqref{eq:reward}.
Following \citet{bogunovic2021stochastic}, we adopt the $\theta$-flipping corruption mechanism.
The adversary continues flipping the sign of $\langle X_t,\theta^*\rangle$
until the cumulative corruption reaches $C$.
Specifically, for corrupted rounds we observe
$r_t = -\langle  X_t , \theta^*\rangle + \eta_t$,
while the remaining rounds are left uncorrupted, i.e., $r_t = r_t^{\prime}$.
In our experiments, we consider the cases $C=0$ and $C=100$.
We investigate two noise models:
\begin{itemize}
  \item \textbf{In the case of $\epsilon=1$:} Following \citet{yucorruption}, we generate the noise from a Student's $t$-distribution with $3$
degrees of freedom.
This distribution has finite variance ($\mathbb{E}[\eta_t^2] = 3$) and satisfies Assumption~\ref{ass:noise} with $\epsilon = 1$ and $\nu_t = \nu = \sqrt{3}$.

  \item \textbf{In the case of $\epsilon<1$:}
  We generate the noise from a Pareto distribution with shape parameter $\alpha=1.5$ and minimum value $x_{\min}=1.$ This distribution has a finite $(1+\epsilon)$-th moment for $\epsilon=0.4$, and satisfies Assumption~\ref{ass:noise} with $\nu_t=\nu = 15^{\frac{1}{1+\epsilon}}$.
\end{itemize}
We first compare ours (CR-Hvt-UCB) with GAdaOFUL~\citep{yucorruption} when $\epsilon=1$ (Figures~\ref{fig:tdist_regret} and~\ref{fig:tdist_regret}).
We next investigate the behavior of ours and GAdaOFUL when $\epsilon<1$ (Figure~\ref{fig:pareto_regret}).
Each experiment is conducted for $T=5000$ and averaged over $10$ independent runs.

As shown in Figure~\ref{fig:tdist_regret}, we observe that ours achieves regret comparable to that of GAdaOFUL
throughout the time horizon $T$, both in the absence of corruption ($C=0$)
and under adversarial corruption ($C=100$). In contrast, our algorithm achieves substantially lower runtime. 
As shown in Figure~\ref{fig:tdist_runtime}, it is several orders of magnitude faster than GAdaOFUL in both settings.

Figure~\ref{fig:pareto_regret} reveals a clear contrast between the two methods.
For GAdaOFUL, the regret exhibits nearly linear growth both
in the absence of corruption ($C=0$) and under corruption ($C=100$).
This indicates that GAdaOFUL fails to achieve sublinear regret
even without corruption in the case of $\epsilon<1$.
In contrast, although corruption worsens the regret,
our algorithm achieves sublinear regret
for both settings.
\section{Conclusion}
\label{sec:conclusion}
In this work, we study linear contextual bandits under the simultaneous presence
of adversarial corruption and heavy-tailed noise. We highlight the limitations of existing approaches in terms of robustness to heavy-tailed noise and
computational efficiency. 
To address these limitations, we propose CR-Hvt-UCB (Algorithm~\ref{alg:CR-HvtUCB}).
The algorithm is designed to be robust to both adversarial corruption and heavy-tailed noise and is computationally efficient via OMD update. The central mechanism underlying our algorithm is the adaptive scaling
$\sigma_t$, which mitigates the effect of adversarial corruption on parameter estimation.
When both the corruption
level $C$ and the $(1+\epsilon)$-moment upper bounds $\nu_t$ are known, our algorithm achieves
$\tilde{\mathcal{O}}(dT^{\frac{1-\epsilon}{2(1+\epsilon)}} \sqrt{\sum_{t=1}^T \nu_t^2} + dT^{\frac{1-\epsilon}{2(1+\epsilon)}}( 1 \vee C))$ regret.
We further provide regret guarantees for the settings where
$C$ or $\nu_t$, or both, are unknown.

In the corruption-free case ($C = 0$), our regret rate is optimal
among algorithms designed for linear contextual bandits
under bounded $(1+\epsilon)$-moment assumptions;
however, it does not attain the lower bound
for heavy-tailed linear bandits~\citep{tajdini2025improved}. Achieving this lower bound remains an important direction for future work.
Another direction is to extend our framework beyond linear
models to GLBs \citep{yucorruption}.

\bibliographystyle{plainnat}
\bibliography{refs}

\appendix
\section{Additional Related Work}
\label{sec:add_related_work}

This appendix provides a broader overview of related work beyond the main comparisons presented in Section~\ref{sec:related_work}.
Rather than focusing only on the closest comparisons,
we summarize existing studies according to three themes:
(1) adversarial corruption,
(2) heavy-tailed noise, and
(3) computational efficiency.

\subsection{Bandit with Adversarial Corruption}

Adversarial corruption was first studied in the multi-armed bandit setting
\citep{lykouris2018stochastic, gupta2019better}.
\citet{lykouris2018stochastic} derived regret bounds in which the total corruption level $C$ appears multiplicatively with the time horizon $T$. In contrast, \citet{gupta2019better} established that the optimal dependence admits an additive decomposition of the form $\mathrm{Reg}(T) = o(T) + \mathcal{O}(C).$
This additive regret bound has since become a standard benchmark for corruption-robust bandit algorithms.
The study of adversarial corruption was later extended to linear bandits
\citep{li2019stochastic, bogunovic2021stochastic}.
\citet{li2019stochastic} and \citet{bogunovic2021stochastic} proposed corruption-robust algorithms based on arm-elimination techniques. While these methods guarantee robustness to adversarial corruption, they fundamentally rely on a fixed and finite decision set, and therefore do not directly extend to the general linear contextual bandit setting considered in this work.
\citet{bogunovic2021stochastic} further studied extensions to linear contextual bandits under additional structural assumptions on the contexts. In contrast, we impose no such structural assumptions.
Subsequent works \citep{ding2022robust, zhao2021linear} proposed OFUL-type algorithms that are robust to adversarial corruption. However, their regret guarantees are not optimal in the corruption-dependent term. In particular, they do not attain the optimal joint dependence on $T$ and the feature dimension $d$ established by \citet{bogunovic2021stochastic}.
A major advance was made by \citet{he2022nearly}, who proposed an OFUL-based algorithm for linear contextual bandits under adversarial corruption and achieved the optimal regret bound $\tilde{\mathcal{O}}(d\sqrt{T} + dC).$
Most existing works on linear contextual bandits under adversarial corruption assume sub-Gaussian noise, with relatively limited results beyond this regime. 
A notable exception is \citet{yucorruption}, which studies GLBs under both adversarial corruption and heavy-tailed noise under a finite-variance assumption. We adopt the same additive corruption model as in \citep{yucorruption, he2022nearly}, while further relaxing the noise assumption.

\subsection{Bandit with Heavy-tailed Noise}

Multi-armed bandits with heavy-tailed rewards were first studied by \citet{bubeck2013bandits} under the assumption that the noise admits a bounded $(1+\epsilon)$-th moment for some $\epsilon \in (0,1]$. They proposed robust algorithms based on truncation and median-of-means (MoM) estimators. These techniques were later extended to linear bandits by \citet{medina2016no}.
Under the same bounded $(1+\epsilon)$-th moment assumption, \citet{shao2018almost} established a lower bound of order $\Omega(dT^{\frac{1}{1+\epsilon}})$ for heavy-tailed linear contextual bandits. This was the best-known lower bound at the time. They also proposed truncation- and MoM-based algorithms achieving regret of the same order.
\citet{xue2023efficient} later achieved the same order of regret in GLBs, also using truncation and MoM techniques to handle heavy-tailed noise.
For linear contextual bandits with a finite decision set, \citet{xue2020nearly} derived matching lower bounds under the same moment assumption and proposed an algorithm attaining the optimal rate in this restricted setting. However, their approach relies on the finiteness of the decision set and does not extend directly to general contextual bandits.
An alternative line of work is based on Huber-type estimators. \citet{li2023variance} proposed an algorithm based on adaptive Huber regression \citep{sun2020adaptive}, assuming finite-variance noise and deriving variance-aware regret guarantees. This direction was further developed by \citet{huang2023tackling} and \citet{wang2025heavy}, who relaxed the finite-variance assumption to bounded $(1+\epsilon)$-th moments. Under this weaker assumption, they obtained regret bounds of order $\tilde{\mathcal{O}}\left(dT^{\frac{1-\epsilon}{2(1+\epsilon)}} \sqrt{\sum_{t=1}^T \nu_t^2} + dT^{\frac{1-\epsilon}{2(1+\epsilon)}}\right)$, where $\nu_t$ denotes an upper bound on the $(1+\epsilon)$-th moment of the noise at round~$t$.
More recently, \citet{tajdini2025improved} derived an improved lower bound of order $\Omega(d^{\frac{2\epsilon}{1+\epsilon}} T^{\frac{1}{1+\epsilon}})$ for linear contextual bandits with heavy-tailed noise, and proposed an algorithm achieving $\tilde{\mathcal{O}}(d^{\frac{1+3\epsilon}{2(1+\epsilon)}} T^{\frac{1}{1+\epsilon}})$. However, their algorithm assumes a fixed decision set and therefore cannot be directly applied to general contextual settings. Consequently, among algorithms applicable to general linear contextual bandits, the order $\tilde{\mathcal{O}}(dT^{\frac{1}{1+\epsilon}})$ remains the best-known regret rate.
Most existing works on linear contextual bandit with heavy-tailed noise under bounded moment assumptions primarily focus on robustness to heavy-tailed noise and do not account for robustness to adversarial corruption. 
An exception is \citet{yucorruption}, which studies GLBs under both adversarial corruption and heavy-tailed noise. 
Their analysis, however, assumes finite-variance noise. 
In contrast, following the bounded $(1+\epsilon)$-moment framework of \citet{huang2023tackling, wang2025heavy}, we relax the noise assumption to bounded $(1+\epsilon)$-th moments while simultaneously ensuring robustness to adversarial corruption within a unified framework.

\subsection{Computational Efficiency}

In linear bandits, the OFUL algorithm of \citet{abbasi2011improved} maintains a
regularized least-squares estimator with a closed-form update, which admits
$O(1)$ computational cost per round.
Similar closed-form estimators are used in OFUL-based methods
for contextual bandits under adversarial corruption
\citep{zhao2021linear, ding2022robust, he2022nearly}.
In linear contextual bandit problems with heavy-tailed noise,
\citet{xue2023efficient} proposed algorithms based on truncation and
MoM estimators that admit $O(1)$ per-round computational cost.
In contrast, algorithms based on Huber-type loss minimization
generally require iterative optimization
\citep{huang2023tackling, li2023variance}.
In particular, \citet{yucorruption} incur a per-round cost of
$O(t \log T)$ in the combined corruption and heavy-tailed setting.
Recent work has shown that OMD
can enable efficient one-pass updates in bandit problems
\citep{zhang2024online, zhang2025generalized}.
Building on this idea, \citet{wang2025heavy}
developed an OMD-based estimator for heavy-tailed linear contextual bandits.
Our work extends this line of research to both heavy-tailed noise and adversarial corruption.
\section{Analysis for the Confidence Interval}
We first introduce the notation used in the proof of Lemma~\ref{lem:conf_interval}. In particular, we formalize the properties of the Huber loss, define the clean (corruption- and noise-free) counterpart of the loss, and specify the scale and threshold parameters that control the effect of heavy-tailed noise and adversarial corruption.
\begin{property}[Property 1 of \citep{huang2023tackling, wang2025heavy}]
Let $f_\tau(x)$ be the Huber loss defined in~\eqref{def:huber_loss1}. 
The following properties hold:
\begin{enumerate}
    \item $\lvert f'_\tau(x) \rvert = \min\{\lvert x \rvert, \tau\}$. 
    \label{property:huber1}

    \item $f'_\tau(x) = \tau f'_1\left(\frac{x}{\tau}\right)$. 
    \label{property:huber2}

    \item For any $\epsilon \in (0,1]$, $f'_1(x)$ satisfies $-\log\left(1 - x + |x|^{1+\epsilon}\right)
    \le f'_1(x)
    \le\log\left(1 + x + |x|^{1+\epsilon}\right).$
    \label{property:huber3}

    \item $f''_\tau(x) = \Indi\{\lvert x \rvert \le \tau\}$. 
    \label{property:huber4}
\end{enumerate}
\label{property:huber}
\end{property}

Recall that $z_t(\theta) := \frac{r_t - \langle \theta, X_t \rangle}{\sigma_t}$
is the normalized prediction error. Based on Property~\ref{property:huber}, the gradient of the loss~\eqref{def:huber_loss2} can be written as
\begin{equation}
\nabla \ell_t(\theta) =
\begin{cases}
- z_t(\theta)\dfrac{X_t}{\sigma_t}, & \text{if } \lvert z_t(\theta) \rvert \le \tau_t, \\
- \tau_t\dfrac{X_t}{\sigma_t},      & \text{if } z_t(\theta) > \tau_t, \\
  \tau_t\dfrac{X_t}{\sigma_t},      & \text{if } z_t(\theta) < -\tau_t.
\end{cases}
\label{def:huber_gra}
\end{equation}
Moreover, the norms of the gradient and the Hessian are given by
\begin{align}
\bigl\|\nabla \ell_t(\theta)\bigr\|_2
&= \left\|\min\bigl\{\lvert z_t(\theta)\rvert,   \tau_t\bigr\}
   \frac{X_t}{\sigma_t}\right\|_2,
\label{equ:grad_norm}\\
\nabla^2 \ell_t(\theta)
&= \Indi\bigl\{\lvert z_t(\theta)\rvert \le \tau_t\bigr\}
   \frac{X_t X_t^\top}{\sigma_t^2}.
\label{equ:hessian}
\end{align}

\paragraph{Clean Huber loss.}We introduce the clean (noise-free
and corruption-free) loss to separate the stochastic noise and adversarial corruption effects,.
We define
\begin{equation*}
    \tilde{z}_t(\theta) :=
    \frac{X_t^\top(\theta^* - \theta)}{\sigma_t},
\end{equation*}
which we call the clean normalized prediction error.
Based on $\tilde{z}_t(\theta)$, we define the corresponding clean 
Huber loss:
\begin{equation*}
\tilde{\ell}_t(\theta) :=
\begin{cases}
\frac{1}{2} \tilde{z}_t(\theta)^2,
   & \text{if } \lvert \tilde{z}_t(\theta)\rvert \le \tau_t, \\
\tau_t \lvert \tilde{z}_t(\theta)\rvert - \frac{\tau_t^2}{2},
   & \text{if } \lvert \tilde{z}_t(\theta)\rvert > \tau_t.
\end{cases}
\end{equation*}

Similarly to $\ell_t(\theta)$, the gradient and Hessian of $\tilde{\ell}_t(\theta)$ take the form
\begin{align}
\nabla \tilde{\ell}_t(\theta)
&=
\begin{cases}
- \tilde{z}_t(\theta)\dfrac{X_t}{\sigma_t},
   & \text{if } \lvert \tilde{z}_t(\theta)\rvert \le \tau_t, \\
- \tau_t\dfrac{X_t}{\sigma_t},
   & \text{if } \tilde{z}_t(\theta) > \tau_t, \\
    \tau_t\dfrac{X_t}{\sigma_t},
   & \text{if } \tilde{z}_t(\theta) < -\tau_t,
\end{cases}
\label{def:huber_gra_clean}\\
\nabla^2 \tilde{\ell}_t(\theta)
&=
\Indi\{\lvert \tilde{z}_t(\theta)\rvert \le \tau_t\}
\frac{X_t X_t^\top}{\sigma_t^2}.
\label{def:huber_hessian_clean}
\end{align}

\paragraph{Huber threshold parameter.}
Recall that the threshold parameter $\tau_t$ is defined as
\begin{equation}
\tau_t
:= \tau_0 \frac{\sqrt{1 + w_t^2}}{w_t}
t^{\frac{1 - \epsilon}{2(1 + \epsilon)}},
\qquad
\tau_0
:= \frac{ \sqrt{2\kappa}
(\log 3T)^{\frac{1 - \epsilon}{2(1 + \epsilon)}} }
{ \left( \log \frac{2T^2}{\delta} \right)^{\frac{1}{1 + \epsilon}} } .
\label{def:tau}
\end{equation}

Here $w_t$ is given by
\begin{equation}
w_t := \frac{1}{\sqrt{\alpha}}
\left\| \frac{X_t}{\sigma_t} \right\|_{V_{t-1}^{-1}} .
\label{def:wt}
\end{equation}

\paragraph{Scale parameter.}
Recall that the scale parameter $\sigma_t$ is defined as
\begin{equation}
\sigma_t 
:= \max\Biggl\{
\nu_t,
 \sigma_{\min},
 \sqrt{\frac{2\beta_{t-1}}
            {\tau_0 \sqrt{\alpha}
             t^{\frac{1-\epsilon}{2(1+\epsilon)}}}}
 \|X_t\|_{V_{t-1}^{-1}}, 
 \sqrt{C}\kappa^{-1/4}
   \|X_t\|_{V_{t-1}^{-1}}^{1/2}
\Biggr\}.
\label{def:sigma}
\end{equation}

\paragraph{Good event.}
To state high-probability guarantees,
we define the following confidence event.
\begin{equation}
\mathcal{A}_t
:= \left\{
\forall s \in [t],  
\left\| \hat{\theta}_s - \theta^* \right\|_{V_{s-1}}
\le \beta_{s-1}
\right\}.
\label{def:At}
\end{equation}

\subsection{Key Lemma}
\label{sec:key_lemma}
In this section, we present three lemmas
that are key to the proof of Lemma~\ref{lem:conf_interval}.
The Lemma~\ref{lem:key_decomp} established by Lemma~3 of~\citet{wang2025heavy},
provides a decomposition of the estimation error into four components.
Based on this decomposition, Lemma~\ref{lem:stability} analyzes the stability term,
while Lemma~\ref{lem:gene_gap} analyzes the generalization-gap term.

\begin{lemma}[Lemma 3 in \citet{wang2025heavy}]
\label{lem:key_decomp}
Suppose the good event $\mathcal{A}_t$ holds, and let $\tau_t$ be
chosen as in~\eqref{def:tau}. 
Assume that $\sigma_t$ satisfies
\begin{equation}
\sigma_t^2
  \ge  
\frac{
2\|X_t\|_{V_{t-1}^{-1}}^2\beta_{t-1}
}{
\sqrt{\alpha}\tau_0
t^{\frac{1-\epsilon}{2(1+\epsilon)}} } .
\label{cond:sigma_key}
\end{equation}
Then the estimation error admits the decomposition
\begin{align}
&\|\hat{\theta}_{t+1}-\theta^*\|_{V_t}^2 \notag \\
&\le
\underbrace{4\lambda S^2}_{\text{(i) regularization term}}
+ \underbrace{\sum_{s=1}^t
\bigl\|\nabla \ell_s(\hat{\theta}_s)\bigr\|_{V_s^{-1}}^2}_{\text{(ii) stability term}}
\notag 
+ \underbrace{2 \sum_{s=1}^t
\left\langle
\nabla\tilde{\ell}_s(\hat{\theta}_s)
- \nabla\ell_s(\hat{\theta}_s),
\hat{\theta}_s-\theta^*
\right\rangle}_{\text{(iii) generalization-gap term}}
\notag \\
&\quad
+ \underbrace{\left(\frac{1}{\alpha}-1\right)
\sum_{s=1}^t
\bigl\|\hat{\theta}_s-\theta^*\bigr\|_{X_sX_s^\top/\sigma_s^2}}_{\text{(iv) negative term}}.
\label{eq:key_decomp}
\end{align}
\end{lemma}

The four terms respectively correspond to 
(i) the regularization term reflecting prior knowledge,
(ii) the stability term, which quantifies how stably the iterates evolve,
(iii) the generalization-gap term, which measures the discrepancy between the clean and corrupted gradients and captures the effect of adversarial corruption and heavy-tailed noise,
and (iv) a negative term characteristic of OMD analyses, involving the estimation error.

The proof follows the same argument as in \citet{wang2025heavy}.
Although our model additionally allows adversarial corruption,
the Lemma~\ref{lem:key_decomp} depends only on the good event $\mathcal{A}_t$,
condition~\eqref{cond:sigma_key}, and the construction of $\tau_t$,
all of which are independent of the adversarial corruption.
Hence, the same argument applies.

\begin{lemma}[Stability Lemma]
\label{lem:stability}
Let $\tau_t$ be defined in~\eqref{def:tau} and choose $\sigma_t$ as
\begin{equation}
\sigma_t
= \max\Biggl\{\nu_t,  \sigma_{\min},  
 2\sqrt{2}\|X_t\|_{V_{t-1}^{-1}},  \sqrt{C}\kappa^{-1/4}
   \|X_t\|_{V_{t-1}^{-1}}^{1/2}
\Biggr\},
\label{def:sigma_stab_en}
\end{equation}
where $\kappa
= d \log\!\left(
1+\frac{L^2 T}{\sigma_{\min}^2 \lambda \alpha d}
\right).$
Let $\alpha > 0$ be a constant stepsize parameter in the OMD update.
Then with probability at least $1-\delta$, for all $t\ge1$,
\begin{align*}
\sum_{s=1}^t \bigl\|\nabla \ell_s(\hat{\theta}_s)\bigr\|^2_{V_s^{-1}} 
&\le
18\alpha
\Biggl[
t^{\frac{1-\epsilon}{2(1+\epsilon)}}
\sqrt{2\kappa}
(\log 3T)^{\frac{1-\epsilon}{2(1+\epsilon)}}
\bigl(\log(2T^2/\delta)\bigr)^{\frac{\epsilon}{1+\epsilon}}
\Biggr]^2 
+ \frac{3}{8}\sum_{s=1}^t
\bigl\|\theta^* - \hat{\theta}_s\bigr\|^2_{\frac{X_s X_s^\top}{\sigma_s^2}}
+ 3\kappa.
\end{align*}
\end{lemma}

\begin{lemma}[Generalization-gap Lemma]
\label{lem:gene_gap}
Let $\tau_t$ be defined in~\eqref{def:tau}. We choose $\sigma_t$ as
\begin{equation}
\sigma_t
=
\max\Biggl\{
  \nu_t,
  \sigma_{\min},
  2\sqrt{2}\|X_t\|_{V_{t-1}^{-1}},
  \sqrt{\frac{
      2 \|X_t\|_{V_{t-1}^{-1}}^2 \beta_{t-1}
  }{
      \sqrt{\alpha}\tau_0
      t^{\frac{1-\epsilon}{2(1+\epsilon)}}
  }},
  \sqrt{C}\kappa^{-1/4}\|X_t\|_{V_{t-1}^{-1}}^{1/2}
\Biggr\}.
\label{def:sigma_gap}
\end{equation}
The event $\mathcal{A}_t$ is defined as in~(\ref{def:At}).
Then, with probability at least $1-3\delta$, for any $t \ge 1$,
\begin{align*}
&\sum_{s=1}^t \left\langle 
\nabla \tilde{\ell}_s(\hat{\theta}_s) - \nabla \ell_s(\hat{\theta}_s), 
\hat{\theta}_s - \theta^*
\right\rangle \Indi_{\mathcal{A}_s} \\
&\le 65 \sqrt{\alpha + \tfrac{1}{8}}
\log\frac{2T^2}{\delta}
\tau_0 t^{\frac{1-\epsilon}{2(1+\epsilon)}} 
\max_{u \in [t+1]} \beta_{u-1} 
+ 12 \sqrt{1+\tfrac{1}{8\alpha}}
\sqrt{\kappa}
t^{\frac{1-\epsilon}{2(1+\epsilon)}} 
\max_{u \in [t+1]} \beta_{u-1} \\
&\quad+ \frac{1}{4}
\left(
\lambda \alpha + 
\sum_{s=1}^t 
\left\langle \frac{X_s}{\sigma_s}, \hat{\theta}_s - \theta^* \right\rangle^2
\right) 
+ \left(
8 t^{\frac{1-\epsilon}{2(1+\epsilon)}}
\sqrt{2\kappa}(\log 3T)^{\frac{1-\epsilon}{2(1+\epsilon)}}
\left( 
\log \frac{2T^2}{\delta}
\right)^{\frac{\epsilon}{1+\epsilon}}
\right)^{2} \\
&\qquad + \sqrt{\kappa}\max_{u\in [t+1]} \beta_{u-1}.
\end{align*}
\end{lemma}

\subsection{Proof of Lemma~\ref{lem:stability}}
\label{sec:stability_proof}
\begin{proof}
We first decompose the squared $V_t^{-1}$-norm of the gradient
by separating the contributions of
$\eta_s$, $X_s^\top(\theta^*-\hat{\theta}_s)$, and $c_s$. For any $t \ge 1$, the squared $V_t^{-1}$-norm of the gradient can be
written as
\begin{align}
&\bigl\|\nabla \ell_t(\hat{\theta}_t)\bigr\|_{V_t^{-1}}^2 \notag \\
&=
\left\|
  \min\left\{
    \left|
      \frac{r_t - X_t^\top \hat{\theta}_t}{\sigma_t}
    \right|,
    \tau_t
  \right\}
  \frac{X_t}{\sigma_t}
\right\|_{V_t^{-1}}^2 \notag \\
&=
\left\|
  \min\left\{
    \left|
      \frac{X_t^\top \theta^* + \eta_t + c_t - X_t^\top \hat{\theta}_t}{\sigma_t}
    \right|,
    \tau_t
  \right\}
  \frac{X_t}{\sigma_t}
\right\|_{V_t^{-1}}^2 \notag \\
&=
\left\|
  \min\left\{
    \left|
      \frac{\eta_t}{\sigma_t}
      + \frac{c_t}{\sigma_t}
      + \frac{X_t^\top(\theta^* - \hat{\theta}_t)}{\sigma_t}
    \right|,
    \tau_t
  \right\}
  \frac{X_t}{\sigma_t}
\right\|_{V_t^{-1}}^2 \notag \\
&\le
\left(
  \min\left\{
    \left|
      \frac{\eta_t}{\sigma_t}
    \right|
    +
    \left|
      \frac{X_t^\top(\theta^* - \hat{\theta}_t)}{\sigma_t}
    \right|
    +
    \left|
      \frac{c_t}{\sigma_t}
    \right|,
    \tau_t
  \right\}
\right)^2
\left\|\frac{X_t}{\sigma_t}\right\|_{V_t^{-1}}^2 \notag \\
&\le
\left(
  \min\left\{\left|\frac{\eta_t}{\sigma_t}\right|,\tau_t\right\}  + \min\left\{\left|\frac{X_t^\top(\theta^*-\hat{\theta}_t)}{\sigma_t}\right|,\tau_t\right\}
  + \min\left\{\left|\frac{c_t}{\sigma_t}\right|,\tau_t\right\}
\right)^2
\left\|\frac{X_t}{\sigma_t}\right\|_{V_t^{-1}}^2 \notag \\
&\le
3\left(\min\left\{\left|\frac{\eta_t}{\sigma_t}\right|,\tau_t\right\}\right)^2
 \left\|\frac{X_t}{\sigma_t}\right\|_{V_t^{-1}}^2 
+ 3\left( \min\left\{\left|\frac{X_t^\top(\theta^*-\hat{\theta}_t)}{\sigma_t}\right|,\tau_t\right\}\right)^2
 \left\|\frac{X_t}{\sigma_t}\right\|_{V_t^{-1}}^2 \notag\\
&\quad + 3\left(\min\left\{\left|\frac{c_t}{\sigma_t}\right|,\tau_t\right\}\right)^2
 \left\|\frac{X_t}{\sigma_t}\right\|_{V_t^{-1}}^2 \notag \\
&\le
3\left(\min\left\{\left|\frac{\eta_t}{\sigma_t}\right|,\tau_t\right\} \right)^2
 \left\|\frac{X_t}{\sigma_t}\right\|_{V_t^{-1}}^2 
 + 3\left(
    \frac{X_t^\top(\theta^*-\hat{\theta}_t)}{\sigma_t}
  \right)^2
  \left\|\frac{X_t}{\sigma_t}\right\|_{V_t^{-1}}^2 + 3\left(\frac{c_t}{\sigma_t}\right)^2
  \left\|\frac{X_t}{\sigma_t}\right\|_{V_t^{-1}}^2 . 
\label{equ:norm_gradient_inequality}
\end{align}
The first equality uses the definition of the gradient in
\eqref{equ:grad_norm}, the second equality uses the observed reward
$r_t = X_t^\top \theta^* + \eta_t + c_t$ in~\eqref{eq:reward}. 
The first inequality follows from the triangle inequality.
The second uses
$\min\{a+b,c\} \le \min\{a,c\} + \min\{b,c\}$ for all
$a,b,c \in \mathbb{R}_+$.
The third uses
$(a+b+c)^2 \le 3(a^2+b^2+c^2)$ for all $a,b,c \in \mathbb{R}$,
and the last uses
$\min\{a,b\} \le a$ for all $a,b \in \mathbb{R}$.

To simplify the expression of the weighted norm $\left\|\frac{X_t}{\sigma_t}\right\|_{V_t^{-1}}$, define
\begin{equation}
\tilde{V}_t
:= \lambda\alpha I_d
 + \sum_{s=1}^t \frac{X_sX_s^\top}{\sigma_s^2}
= \alpha V_t.
\label{def:tilde_V}
\end{equation}
Using $\tilde{V}_t = \alpha V_t$, the quantity $w_t$ defined in~\eqref{def:wt} can be rewritten as
\begin{equation}
w_t
= \frac{1}{\sqrt{\alpha}}
  \left\|\frac{X_t}{\sigma_t}\right\|_{V_{t-1}^{-1}}
= \left\|\frac{X_t}{\sigma_t}\right\|_{\tilde{V}_{t-1}^{-1}} .
\label{equ:w_t_tilde_V}
\end{equation}
Since $\tilde{V}_t = \tilde{V}_{t-1} + X_tX_t^\top/\sigma_t^2$, applying the
Sherman--Morrison--Woodbury formula~\citep{horn2012matrix} gives
\begin{align*}
\tilde{V}_t^{-1} 
=
\left(
  \tilde{V}_{t-1}
  + \frac{X_tX_t^\top}{\sigma_t^2}
\right)^{-1} 
&=
\tilde{V}_{t-1}^{-1}
- \tilde{V}_{t-1}^{-1}
  \frac{X_t}{\sigma_t}
  \left(
    1 + \frac{X_t^\top}{\sigma_t}\tilde{V}_{t-1}^{-1}\frac{X_t}{\sigma_t}
  \right)^{-1}
  \frac{X_t^\top}{\sigma_t}\tilde{V}_{t-1}^{-1} \\
&=
\tilde{V}_{t-1}^{-1}
- \frac{\tilde{V}_{t-1}^{-1} X_t X_t^\top \tilde{V}_{t-1}^{-1}}
       {\sigma_t^2 (1 + w_t^2)} .
\end{align*}
Then,
\begin{align*}
\left\|\frac{X_t}{\sigma_t}\right\|_{\tilde{V}_t^{-1}}^2
=
\frac{X_t^\top \tilde{V}_t^{-1} X_t}{\sigma_t^2} 
&=
\frac{1}{\sigma_t^2}
X_t^\top\left(
  \tilde{V}_{t-1}^{-1}
  - \frac{\tilde{V}_{t-1}^{-1} X_t X_t^\top \tilde{V}_{t-1}^{-1}}
         {\sigma_t^2(1+w_t^2)}
\right) X_t 
= w_t^2 - \frac{w_t^4}{1+w_t^2}
 = \frac{w_t^2}{1+w_t^2} .
\end{align*}
Recalling that~\eqref{equ:w_t_tilde_V}, we obtain
\begin{equation}
\left\|\frac{X_t}{\sigma_t}\right\|_{V_t^{-1}}^2
= \alpha\frac{w_t^2}{1+w_t^2}.
\label{equ:w_t_rewritten}
\end{equation}

Plugging~\eqref{equ:w_t_rewritten} into~\eqref{equ:norm_gradient_inequality} and summing over
$s=1,\dots,t$, we can obtain
\begin{align*}
\sum_{s=1}^t
\bigl\|\nabla \ell_s(\hat{\theta}_s)\bigr\|_{V_s^{-1}}^2 
&\le
\underbrace{
  3\alpha \sum_{s=1}^t
  \left(\min\left\{\left|\frac{\eta_s}{\sigma_s}\right|,\tau_s\right\} \right)^2
  \frac{w_s^2}{1+w_s^2}
}_{\text{TERM (A.1)}} \\
&\quad
+ \underbrace{
  3\alpha \sum_{s=1}^t
  \left(
    \frac{X_s^\top(\theta^*-\hat{\theta}_s)}{\sigma_s}
  \right)^2
  \frac{w_s^2}{1+w_s^2}
}_{\text{TERM (A.2)}} 
+ \underbrace{
  3\alpha \sum_{s=1}^t
  \left(\frac{c_s}{\sigma_s}\right)^2
  \frac{w_s^2}{1+w_s^2}
}_{\text{TERM (A.3)}} .
\end{align*}
TERM (A.1) is the stochastic term corresponding to the Huber gradient induced by the noise $\eta_s$, 
TERM (A.2) is the deterministic term corresponding to the estimation error 
$\theta^* - \hat{\theta}_s$, and 
TERM (A.3) is the term induced by the corruption $c_s$.
We now control each of the above terms separately.
\paragraph{Bound on TERM (A.1).}
This term can be bounded by the same argument as in
\citet{wang2025heavy}.
Applying Lemma~\ref{lem:huang_C.5} under our choice of $\tau_t$ in~\eqref{def:tau} and $b:=\max_{t\in[T]}\frac{\nu_t}{\sigma_t} \le 1,$
we obtain the following bound that holds with probability at least
$1-\delta$ for all $t\ge1$:
\begin{align}
\label{eq:A1_prelimit}
\sum_{s=1}^t
\left(\min\left\{\left|\frac{\eta_s}{\sigma_s}\right|,\tau_s\right\} \right)^2
\frac{w_s^2}{1+w_s^2} \le
t^{\frac{1-\epsilon}{1+\epsilon}}
\Bigl(
  \sqrt{\tau_0^{1-\epsilon}
         (2\kappa)^{\frac{1+\epsilon}{2}}
         (\log 3T)^{\frac{1-\epsilon}{2}}}
         +
    \tau_0\sqrt{2\log(2T^2/\delta)}
\Bigr)^{2}.
\end{align}
With the choice of $\tau_0$ in~\eqref{def:tau},
we observe that
$\sqrt{
  \tau_0^{1-\epsilon}
  (2\kappa)^{\frac{1+\epsilon}{2}}
  (\log 3T)^{\frac{1-\epsilon}{2}}
}
= \tau_0 \sqrt{\log(2T^2/\delta)}.
$
Therefore, both terms inside the parentheses of
\eqref{eq:A1_prelimit} coincide, and hence
\begin{align}
\label{equ:termA1}
\sum_{s=1}^t
\left(\min\left\{\left|\frac{\eta_s}{\sigma_s}\right|,\tau_s\right\} \right)^2
\frac{w_s^2}{1+w_s^2} \notag 
&\le
t^{\frac{1-\epsilon}{1+\epsilon}}
(1+\sqrt{2})^2\Bigl(
\sqrt{2\kappa}
(\log 3T)^{\frac{1-\epsilon}{2(1+\epsilon)}}
(\log(2T^2/\delta))^{\frac{\epsilon}{1+\epsilon}}
\Bigr)^{2} \notag \\
&\le
6t^{\frac{1-\epsilon}{1+\epsilon}}
\Bigl(
\sqrt{2\kappa}
(\log 3T)^{\frac{1-\epsilon}{2(1+\epsilon)}}
(\log(2T^2/\delta))^{\frac{\epsilon}{1+\epsilon}}
\Bigr)^{2}.
\end{align}
\paragraph{Bound on TERM (A.2).}
This term can be bounded in the same way as in
\citet{wang2025heavy}.
By the choice of $\sigma_t$ in~\eqref{def:sigma_stab_en}, we have
$\sigma_t \ge 2\sqrt{2}\|X_t\|_{V_{t-1}^{-1}}$ for all $t$, and hence
\begin{align}
\alpha w_t^2
= \left\|\frac{X_t}{\sigma_t}\right\|_{V_{t-1}^{-1}}^2 = \left\|X_t\right\|_{V_{t-1}^{-1}}^2 \frac{1}{\sigma_t^2}
\le \|X_t\|_{V_{t-1}^{-1}}^2
     \frac{1}{\bigl(2\sqrt{2}\|X_t\|_{V_{t-1}^{-1}}\bigr)^2} 
 = \frac{1}{8} .
\label{eq:1/8}
\end{align}
Since $\frac{1}{1+w_s^2} \le 1$ and $\alpha w_s^2\leq \frac{1}{8}$ for all $s$, we obtain
\begin{align}
3\alpha \sum_{s=1}^t
\left(
  \frac{X_s^\top(\theta^*-\hat{\theta}_s)}{\sigma_s}
\right)^2
\frac{w_s^2}{1+w_s^2} 
&\le
\frac{3}{8} \sum_{s=1}^t
\left(
  \frac{X_s^\top(\theta^*-\hat{\theta}_s)}{\sigma_s}
\right)^2\notag  \\ 
&=
\frac{3}{8} \sum_{s=1}^t
\bigl\|\theta^*-\hat{\theta}_s\bigr\|^2_{\frac{X_sX_s^\top}{\sigma_s^2}} .
\label{equ:termA2}
\end{align}
\paragraph{Bound on TERM (A.3).} 
Using the definition of $\sigma_s$ in \eqref{def:sigma_stab_en} and the fact that $\frac{1}{1+w_s^2}\le 1$, we obtain
\begin{align}
3\alpha
\sum_{s=1}^t
\left(\frac{c_s}{\sigma_s}\right)^{2}
\frac{w_s^2}{1+w_s^2} 
&\le
3\alpha
\sum_{s=1}^t
\frac{c_s^2}{\sigma_s^2}w_s^2
\notag\\
&=
3\alpha
\sum_{s=1}^t
c_s^2
\frac{1}{\sigma_s^4}
\bigl(\sigma_s^2 w_s^2\bigr)
\notag\\
&\le
3\alpha
\sum_{s=1}^t
c_s^2
\frac{\kappa}{C^2\|X_s\|_{V_{s-1}^{-1}}^2}
\bigl(\sigma_s^2 w_s^2\bigr)
\notag\\
&=
3
\sum_{s=1}^t
c_s^2
\frac{\kappa}{C^2}
\frac{1}{\|\frac{X_s}{\sigma_s}\|_{V_{s-1}^{-1}}^2}
w_s^2
\notag\\
&=
3\frac{\kappa}{C^2}
\sum_{s=1}^t
c_s^2,
\label{eq:termA3_pre}
\end{align}
where the first inequality uses $1/(1+w_s^2)\le 1$,
the second inequality follows from the definition of $\sigma_s$ in~\eqref{def:sigma_stab_en}, which ensures that $\sigma_s \ge \sqrt{C}\kappa^{-1/4}\|X_s\|_{V_{s-1}^{-1}}^{1/2},$
and the last equality uses $\|X_s/\sigma_s\|_{V_{s-1}^{-1}}^2=\alpha^{-1}w_s^2$, which follows from the definition of $w_s$ in~\eqref{equ:w_t_tilde_V}.

Since the corruption level $C$ satisfies
\[
C^2
=
\Bigl(\sum_{s=1}^t |c_s|\Bigr)^2
  \ge  
\sum_{s=1}^t c_s^2,
\]
we conclude from \eqref{eq:termA3_pre} that
\begin{equation}
\label{equ:termA3}
3\alpha
\sum_{s=1}^t
\left(\frac{c_s}{\sigma_s}\right)^2
\frac{w_s^2}{1+w_s^2}
  \le  
3\kappa.
\end{equation}

Combining \eqref{equ:termA1}, \eqref{equ:termA2}, and \eqref{equ:termA3}, we obtain that,
with probability at least $1-\delta$, for all $t\ge 1$,
\begin{align*}
\sum_{s=1}^t \bigl\|\nabla \ell_s(\hat{\theta}_s)\bigr\|^2_{V_s^{-1}} 
\le
18\alpha
\Biggl[
t^{\frac{1-\epsilon}{2(1+\epsilon)}}
\sqrt{2\kappa}
(\log 3T)^{\frac{1-\epsilon}{2(1+\epsilon)}}
\bigl(\log(2T^2/\delta)\bigr)^{\frac{\epsilon}{1+\epsilon}}
\Biggr]^2 
+ \frac{3}{8}\sum_{s=1}^t
\bigl\|\theta^* - \hat{\theta}_s\bigr\|^2_{\frac{X_s X_s^\top}{\sigma_s^2}}
+ 3\kappa.
\end{align*}
\end{proof}

\subsection{Proof of Lemma~\ref{lem:gene_gap}}
\label{sec:gene_gap_proof}
\begin{proof}
Following~\citet{wang2025heavy}, we decompose the generalization-gap term as
\begin{align*}
&\left\langle 
\nabla \tilde{\ell}_t(\hat{\theta}_t) - \nabla \ell_t(\hat{\theta}_t), 
\hat{\theta}_t - \theta^*
\right\rangle  \Indi_{\mathcal{A}_t} \\
&=
\left\langle 
\nabla \tilde{\ell}_t(\hat{\theta}_t) + \nabla \ell_t(\theta^*) 
- \nabla \ell_t(\theta^*) - \nabla \ell_t(\hat{\theta}_t),
\hat{\theta}_t - \theta^*
\right\rangle \Indi_{\mathcal{A}_t} \\
&=
\underbrace{
\left\langle 
\nabla \tilde{\ell}_t(\hat{\theta}_t) 
+ \nabla \ell_t(\theta^*)  - \nabla \ell_t(\hat{\theta}_t),
\hat{\theta}_t - \theta^*
\right\rangle
\Indi_{\mathcal{A}_t}}_{\text{TERM (B.1)}} 
+
\underbrace{
\left\langle 
-\nabla \ell_t(\theta^*),
\hat{\theta}_t - \theta^*
\right\rangle 
\Indi_{\mathcal{A}_t}}_{\text{TERM (B.2)}}.
\end{align*}
Since $\nabla \tilde{\ell}_t(\theta^*) = 0$, 
TERM (B.1) compares the gradient increments 
between $\theta^*$ and $\hat{\theta}_t$, 
namely 
$\nabla \tilde{\ell}_t(\hat{\theta}_t) - \nabla \tilde{\ell}_t(\theta^*)$
and 
$\nabla \ell_t(\hat{\theta}_t) - \nabla \ell_t(\theta^*)$. Hence, TERM (B.1) measures how adversarial corruption and heavy-tailed noise distort these gradient increments.
TERM (B.2) corresponds to the gradient evaluated at $\theta^*$,
capturing the direct effect of adversarial corruption and heavy-tailed noise.
Since adversarial corruption enters both terms through the gradient $\nabla \ell_t(\cdot)$, 
we further decompose each term. 
We now analyze TERM (B.1) and TERM (B.2) separately.
\paragraph{Analysis of TERM (B.1).}
To isolate the nonlinearity of the Huber loss,
we introduce its derivative $h_t(\cdot)$, defined for any $z\in\mathbb{R}$ as
\begin{equation}
h_t(z)
=
\begin{cases}
z, & |z| \le \tau_t,\\
\tau_t, & z > \tau_t,\\
-\tau_t, & z < -\tau_t.
\end{cases}
\label{def:psi}
\end{equation}
With this notation, the gradients of the loss~\eqref{def:huber_gra} and the clean loss~\eqref{def:huber_gra_clean} can be written as
\begin{equation}
\nabla \ell_t(\theta)
= -\frac{h_t(z_t(\theta))X_t}{\sigma_t},
\qquad
\nabla \tilde{\ell}_t(\theta)
= -\frac{h_t(z_t(\theta)-z_t(\theta^*))X_t}{\sigma_t}.
\label{equ:ell_rewrite_h_}
\end{equation}
Substituting these expressions into TERM (B.1), we obtain
\begin{align}
&\left\langle 
\nabla \tilde{\ell}_t(\hat{\theta}_t)
+ \nabla \ell_t(\theta^*)
- \nabla \ell_t(\hat{\theta}_t),
\hat{\theta}_t-\theta^*
\right\rangle \Indi_{\mathcal{A}_t} \notag \\
&=
\bigl[
h_t(z_t(\hat{\theta}_t))
- h_t(z_t(\theta^*))
- h_t(z_t(\hat{\theta}_t)-z_t(\theta^*))
\bigr]
\frac{\langle X_t, \hat{\theta}_t - \theta^* \rangle}{\sigma_t}
\Indi_{\mathcal{A}_t}.
\label{equ:h_t_difference}
\end{align}
The linear factor
$\langle X_t,\hat{\theta}_t-\theta^*\rangle/\sigma_t$
can be controlled using the condition of $\mathcal{A}_t$. Hence, we focus on analyzing the inner difference term
$h_t(z_t(\hat{\theta}_t))
- h_t(z_t(\theta^*))
- h_t(z_t(\hat{\theta}_t)-z_t(\theta^*)).$

Under the event $\mathcal{A}_t$, we have 
\begin{align*}
|z_t(\hat{\theta}_t)-z_t(\theta^*)|
= \left|\frac{X_t^\top(\hat{\theta}_t-\theta^*)}{\sigma_t}\right|
&\leq \left\|\frac{X_t}{\sigma_t}\right\|_{V_{t-1}^{-1}}\left\|\hat{\theta}_t-\theta^*\right\|_{V_{t-1}} \\
&\leq \frac{\tau_0 t^{\frac{1-\epsilon}{2(1+\epsilon)}}}{2 w_t \beta_{t-1}}\beta_{t-1} 
\leq \frac{1}{2}\tau_0\frac{\sqrt{1+w_t^2}}{w_t} t^{\frac{1-\epsilon}{2(1+\epsilon)}}
=\frac{\tau_t}{2},
\end{align*} which follows from Cauchy–Schwarz and the definitions of $\sigma_t$ and $\tau_t$.
Hence,
$h_t(z_t(\hat{\theta}_t)-z_t(\theta^*))
= z_t(\hat{\theta}_t)-z_t(\theta^*)$.
Since the value of
\[
h_t(z_t(\hat{\theta}_t))
- h_t(z_t(\theta^*))
- (z_t(\hat{\theta}_t)-z_t(\theta^*))
\]
depends on the magnitude of $z_t(\theta^*)$,
we distinguish two cases.

When $|z_t(\theta^*)|\le \tau_t/2$, we have
\[
|z_t(\hat{\theta}_t)|
\le |z_t(\hat{\theta}_t)-z_t(\theta^*)| + |z_t(\theta^*)|
\le \tau_t,
\]
and also $|z_t(\hat{\theta}_t)-z_t(\theta^*)|\le \tau_t/2<\tau_t$.
Hence all the arguments of $h_t$ lie in its linear region, i.e., $h_t(u)=u$ whenever $|u|\le \tau_t$. Therefore the difference equals zero:
\begin{equation}
h_t(z_t(\hat{\theta}_t))
- h_t(z_t(\theta^*))
- h_t(z_t(\hat{\theta}_t)-z_t(\theta^*))
= 0.
\label{small-zt}
\end{equation}

In general, regardless of the values of $z_t(\theta^*)$, the uniform bound
$|h_t(u)| \le \tau_t$ for all $u \in \mathbb{R}$ implies
\begin{align}
&\left|
h_t(z_t(\hat{\theta}_t))
- h_t(z_t(\theta^*))
- h_t(z_t(\hat{\theta}_t)-z_t(\theta^*))
\right| \notag \\
&\le
\left|h_t(z_t(\hat{\theta}_t))\right|
+ \left|h_t(z_t(\theta^*))\right|
+ \left|h_t(z_t(\hat{\theta}_t)-z_t(\theta^*))\right|
\le 3\tau_t.
\label{big-zt}
\end{align}

Applying \eqref{small-zt} and \eqref{big-zt} to~\eqref{equ:h_t_difference}, we obtain the following
bound for TERM (B.1):
\begin{align*}
&\left\langle 
\nabla \tilde{\ell}_t(\hat{\theta}_t)
+ \nabla \ell_t(\theta^*)
- \nabla \ell_t(\hat{\theta}_t),
\hat{\theta}_t - \theta^*
\right\rangle \Indi_{\mathcal{A}_t} \\
&=
\bigl[
h_t(z_t(\hat{\theta}_t))
- h_t(z_t(\theta^*))
- h_t(z_t(\hat{\theta}_t)-z_t(\theta^*))
\bigr]
\frac{
\langle X_t, \hat{\theta}_t - \theta^* \rangle
}{\sigma_t}
\Indi_{\mathcal{A}_t} \\
&\le
\bigl|
h_t(z_t(\hat{\theta}_t))
- h_t(z_t(\theta^*))
- h_t(z_t(\hat{\theta}_t)-z_t(\theta^*))
\bigr|
\left\|\frac{X_t}{\sigma_t}\right\|_{V_{t-1}^{-1}}
\left\|\hat{\theta}_t - \theta^*\right\|_{V_{t-1}}
\Indi_{\mathcal{A}_t} \\
&\le
\Indi\left\{|z_t(\theta^*)|\le \frac{\tau_t}{2}\right\}\cdot 0 
+
\Indi\left\{|z_t(\theta^*)|> \frac{\tau_t}{2}\right\}
3\tau_t
\left\|\frac{X_t}{\sigma_t}\right\|_{V_{t-1}^{-1}}
\left\|\hat{\theta}_t - \theta^*\right\|_{V_{t-1}} \\
&=
\Indi\left\{|z_t(\theta^*)|> \frac{\tau_t}{2}\right\}
3\tau_0 \frac{\sqrt{1+w_t^2}}{w_t}
t^{\frac{1-\epsilon}{2(1+\epsilon)}}
\sqrt{\alpha} w_t
\left\|\hat{\theta}_t - \theta^*\right\|_{V_{t-1}} \\
&\le
\Indi\left\{|z_t(\theta^*)|> \frac{\tau_t}{2}\right\}
3\tau_0 \sqrt{\alpha + \alpha w_t^2}
t^{\frac{1-\epsilon}{2(1+\epsilon)}}
\beta_{t-1} \\
&\le
\Indi\left\{|z_t(\theta^*)|> \frac{\tau_t}{2}\right\}
3\sqrt{\alpha+\tfrac18}
\tau_0 t^{\frac{1-\epsilon}{2(1+\epsilon)}}
\beta_{t-1}.
\end{align*}
The first inequality follows from Cauchy--Schwarz, the second from
\eqref{small-zt} and \eqref{big-zt}, and the third from the condition of $\mathcal{A}_t$, and the last from \eqref{eq:1/8}.

Summing over $s=1,\dots,t$ and using the monotonicity of
$s^{\frac{1-\epsilon}{2(1+\epsilon)}}$ together with the fact that
$\max_{u\in[t]}\beta_{u-1}\le \max_{u\in[t+1]}\beta_{u-1}$, we obtain
\begin{align*}
&\sum_{s=1}^t 
\left\langle 
\nabla \tilde{\ell}_s(\hat{\theta}_s)
+ \nabla \ell_s(\theta^*)
- \nabla \ell_s(\hat{\theta}_s),
\hat{\theta}_s - \theta^*
\right\rangle \Indi_{\mathcal{A}_s} \\
&\leq \sum_{s=1}^t \Indi\left\{|z_s(\theta^*)|> \frac{\tau_s}{2}\right\}
3\sqrt{\alpha+\tfrac18}
\tau_0 s^{\frac{1-\epsilon}{2(1+\epsilon)}}
\beta_{s-1} \\
&\le
3\sqrt{\alpha+\tfrac{1}{8}}
\tau_0
t^{\frac{1-\epsilon}{2(1+\epsilon)}}
\max_{u\in[t+1]}\beta_{u-1}
\sum_{s=1}^t
\Indi\left\{|z_s(\theta^*)|>\frac{\tau_s}{2}\right\}.
\end{align*}

Next, since $z_s(\theta^*) = (\eta_s + c_s)/\sigma_s$ contains both the heavy-tailed noise and the corruption term, we further decompose the indicator.
Using the union bound, namely that for any real numbers $x,y$ and $z>0$,
$
\Indi\{|x+y|>z\}\le \Indi\{|x|>z/2\}+\Indi\{|y|>z/2\},
$
we obtain
\begin{align*}
\sum_{s=1}^t \Indi\left\{|z_s(\theta^*)|>\frac{\tau_s}{2}\right\}
&=
\sum_{s=1}^t 
\Indi\left\{\left|\frac{\eta_s}{\sigma_s}+\frac{c_s}{\sigma_s}\right| > \frac{\tau_s}{2}\right\} \\
&\le 
\underbrace{
\sum_{s=1}^t 
\Indi\left\{\left|\frac{\eta_s}{\sigma_s}\right|>\frac{\tau_s}{4}\right\}
}_{\text{TERM (B.1.1)}}
+
\underbrace{
\sum_{s=1}^t 
\Indi\left\{\left|\frac{c_s}{\sigma_s}\right|>\frac{\tau_s}{4}\right\}
}_{\text{TERM (B.1.2)}}.
\end{align*}
TERM (B.1.1) represents the stochastic contribution from the heavy-tailed noise $\eta_s$, 
whereas TERM (B.1.2) represents the adversarial contribution due to the corruption term $c_s$.
\paragraph{Bound on TERM (B.1.1).}
By applying Lemma~\ref{lem:c_12} with the threshold $\tau_s/4$ in place of $\tau_s/2$, 
we obtain that, with probability at least $1-\delta$, for all $t\geq 1$,
\begin{equation}
\sum_{s=1}^t 
\Indi\left\{\left|\frac{\eta_s}{\sigma_s}\right| >
\frac{\tau_s}{4}\right\}
\le
\frac{65}{3}\log\frac{2T^2}{\delta}.
\label{equ:termB1_1}
\end{equation}

\paragraph{Bound on TERM (B.1.2).}
We bound the exceedance count
directly:
\begin{align}
\sum_{s=1}^t 
\Indi\left\{\left|\frac{c_s}{\sigma_s}\right| >
\frac{\tau_s}{4}\right\}
&=
\sum_{s=1}^t 
\Indi\left\{
\frac{4}{\tau_s}\left|\frac{c_s}{\sigma_s}\right| > 1
\right\}
\notag \\
&\le 
\sum_{s=1}^t 
\frac{4}{\tau_s}\frac{|c_s|}{\sigma_s} \notag \\
&=
4\sum_{s=1}^t
\frac{1}{\tau_0}
\frac{w_s}{\sqrt{1+w_s^2}}s^{-\frac{1-\epsilon}{2(1+\epsilon)}}
\frac{|c_s|}{\sigma_s^2}
\sigma_s \notag \\
&\le
4\sqrt{\kappa}\frac{1}{\sqrt{\alpha}}\frac{1}{\tau_0}
\sum_{s=1}^t 
|c_s|
\frac{1}{C}
w_s \frac{\sqrt{\alpha}}{\|X_s/ \sigma_s\|_{V_{s-1}^{-1}}}
\notag \\
&=
4\frac{1}{\sqrt{\alpha}}\frac{1}{\tau_0}
\sum_{s=1}^t |c_s| \frac{\sqrt{\kappa}}{C} \notag \\
&=
4\frac{1}{\sqrt{\alpha}}\frac{\sqrt{\kappa}}{\tau_0}.
\label{equ:termB1_2}
\end{align}
The first inequality uses the elementary fact that 
$\Indi\{u>1\}\le u$ for any $u\ge 0$.  
The first equality that follows comes from the definition of 
$\tau_s$ in~\eqref{def:tau}.  
For the second inequality, we use 
$s^{-(1-\epsilon)/(2(1+\epsilon))}\le 1$ and 
$\frac{1}{\sqrt{1+w_s^2}}\le 1$, together with the lower bound 
$\sigma_s\ge \sqrt{C}\kappa^{-1/4}\|X_s\|_{V_{s-1}^{-1}}^{1/2}$ 
from~\eqref{def:sigma_gap}.  

Combining the bounds \eqref{equ:termB1_1} and~\eqref{equ:termB1_2}, we conclude that,
with probability at least $1-\delta$, for all $t\in[T]$,
\begin{align}
&\sum_{s=1}^t 
\left\langle 
\nabla \tilde{\ell}_s(\hat{\theta}_s)
+ \nabla \ell_s(\theta^*)
- \nabla \ell_s(\hat{\theta}_s),
\hat{\theta}_s - \theta^*
\right\rangle \Indi_{\mathcal{A}_s} \notag \\
&\le
3\sqrt{\alpha+\tfrac18}
\tau_0
t^{\frac{1-\epsilon}{2(1+\epsilon)}}
\max_{u\in[t+1]}\beta_{u-1}
\left(
\frac{65}{3}\log\frac{2T^2}{\delta}
+
4\frac{1}{\sqrt{\alpha}}\frac{\sqrt{\kappa}}{\tau_0}
\right) \notag \\
&=
65\sqrt{\alpha+\tfrac18}
\left(\log\frac{2T^2}{\delta}\right)
\tau_0
t^{\frac{1-\epsilon}{2(1+\epsilon)}} 
\max_{u\in[t+1]}\beta_{u-1} 
+
12\sqrt{1+\tfrac{1}{8\alpha}}
\sqrt{\kappa}
t^{\frac{1-\epsilon}{2(1+\epsilon)}} 
\max_{u\in[t+1]}\beta_{u-1}.
\label{equ:termB1}
\end{align}
\paragraph{Analysis of TERM (B.2).}
Recall from~\eqref{equ:ell_rewrite_h_}, the self-regularization term satisfies the following inequality:
\begin{align}
\sum_{s=1}^t
\left\langle 
-\nabla \ell_s(\theta^*), \hat{\theta}_s - \theta^*
\right\rangle 
\Indi_{\mathcal{A}_s}
&=
\sum_{s=1}^t h_s(z_s(\theta^*))
\left\langle \frac{X_s}{\sigma_s},\hat{\theta}_s-\theta^*
\right\rangle 
\Indi_{\mathcal{A}_s} \notag \\
&\le
\left|
\sum_{s=1}^t 
h_s(z_s(\theta^*))
\left\langle \frac{X_s}{\sigma_s},\hat{\theta}_s-\theta^*
\right\rangle
\right|.
\label{equ:regu_first_inequality}
\end{align}
To separate the corruption term $c_t$ from the stochastic noise $\eta_t$
inside $h_t(z_t(\theta^*))$, we exploit the Lipschitz continuity of $h_t$. 

Since $h_t(\cdot)$ is $1$-Lipschitz, it satisfies $|h_t(x)-h_t(y)| \le |x-y|  \quad\text{for all } x,y\in\mathbb{R}.$
Applying this property to 
$x = (\eta_t + c_t)/\sigma_t$ and $y = \eta_t/\sigma_t$ yields
\begin{equation}
\label{equ:lipschitz}
h_t(z_t(\theta^*))=h_t\left(\frac{\eta_t+c_t}{\sigma_t}\right)
\le 
h_t\left(\frac{\eta_t}{\sigma_t}\right)
+ \left|\frac{c_t}{\sigma_t}\right|.
\end{equation}
Substituting~\eqref{equ:lipschitz} into~\eqref{equ:regu_first_inequality} gives
\begin{align}
\sum_{s=1}^t
\left\langle 
-\nabla \ell_s(\theta^*), \hat{\theta}_s - \theta^*
\right\rangle 
\Indi_{\mathcal{A}_s} 
&\le
\left|
\sum_{s=1}^t
\left\{
h_s\left(\frac{\eta_s}{\sigma_s}\right)
+
\left|\frac{c_s}{\sigma_s}\right|
\right\}
\left\langle \frac{X_s}{\sigma_s},\hat{\theta}_s-\theta^*
\right\rangle
\right| \notag \\
&\le
\underbrace{
\left|
\sum_{s=1}^t
h_s\left(\frac{\eta_s}{\sigma_s}\right)
\left\langle \frac{X_s}{\sigma_s},\hat{\theta}_s-\theta^*
\right\rangle
\right|
}_{\text{TERM (B.2.1)}}
+
\underbrace{
\left|
\sum_{s=1}^t
\left|\frac{c_s}{\sigma_s}\right|
\left\langle \frac{X_s}{\sigma_s},\hat{\theta}_s-\theta^*
\right\rangle
\right|
}_{\text{TERM (B.2.2)}},
\label{equ:termB2_before}
\end{align}
where the first inequality follows from \eqref{equ:lipschitz},
and the second from the triangle inequality. TERM (B.2.1) corresponds to the stochastic noise component,
while TERM (B.2.2) captures the bias term induced by $c_s$.
We analyze them separately.
\paragraph{Bound on TERM (B.2.1).}
This term can be handled in the same way as in \citet{wang2025heavy}.
Applying Lemma~\ref{lem:huang_C.2} under our choice of $\tau_t$ in~\eqref{def:tau} and $b:=\max_{t\in[T]}\frac{\nu_t}{\sigma_t} \le 1,$ we obtain that, with probability at least $1-2\delta$,
the following holds for all $t \ge 1$:
\begin{equation}
\left\|\sum_{s=1}^t h_t\left(\frac{\eta_s}{\sigma_s}\right)\frac{X_s}{\sigma_s}\right\|_{\tilde{V}_t^{-1}}
\leq 
\Bigl(
8 t^{\frac{1-\epsilon}{2(1+\epsilon)}}
\sqrt{2\kappa}
(\log 3T)^{\frac{1-\epsilon}{2(1+\epsilon)}}
\bigl(\log\tfrac{2T^2}{\delta}\bigr)^{\frac{\epsilon}{1+\epsilon}}
\Bigr)^2,
\end{equation}
where $\tilde{V}_t$ is defined in~\eqref{def:tilde_V}.
As in \citet{wang2025heavy}, this vector bound reduces to the
scalar form corresponding to
$Z_s = \langle X_s, \hat\theta_s - \theta^*\rangle$, yielding
\begin{align}
&\left|
\sum_{s=1}^t h_s\left(\frac{\eta_s}{\sigma_s}\right)\frac{Z_s}{\sigma_s}
\right|  \notag \\
&\le
\frac14\left(\lambda\alpha + \sum_{s=1}^t \frac{Z_s^2}{\sigma_s^2}\right)
+
\Bigl(
8 t^{\frac{1-\epsilon}{2(1+\epsilon)}}
\sqrt{2\kappa}
(\log 3T)^{\frac{1-\epsilon}{2(1+\epsilon)}}
\bigl(\log\tfrac{2T^2}{\delta}\bigr)^{\frac{\epsilon}{1+\epsilon}}
\Bigr)^2.
\label{equ:termB2_1}
\end{align}
\paragraph{Bound on TERM (B.2.2).}
By triangle inequality and Cauchy--Schwarz, we obtain
\begin{align}
\left|
\sum_{s=1}^t 
\left|\frac{c_s}{\sigma_s}\right|
\left\langle \frac{X_s}{\sigma_s},
\hat{\theta}_s - \theta^* \right\rangle
\right|
&\le
\sum_{s=1}^t 
\frac{|c_s|}{\sigma_s}
\left\|\frac{X_s}{\sigma_s}\right\|_{V_{s-1}^{-1}}
\left\|\hat{\theta}_s - \theta^* \right\|_{V_{s-1}}     \notag       \\
&\le
\sum_{s=1}^t 
\frac{|c_s|}{\sigma_s^2}
\|X_s\|_{V_{s-1}^{-1}}\beta_{s-1} \notag\\
&\le
\sqrt{\kappa}\sum_{s=1}^t 
\frac{|c_s|}{C}\beta_{s-1} \notag\\
&\le
\sqrt{\kappa}
\max_{u\in[t+1]} \beta_{u-1},
\label{equ:termB2_2}
\end{align}
where we used the definition of $\mathcal{A}_s$ in the
second inequality, and the third inequality follows from $\sigma_s \ge \sqrt{C}\kappa^{-1/4}\|X_s\|_{V_{s-1}^{-1}}^{1/2}$ in~\eqref{def:sigma_gap}, and the last follows from  $\beta_{s-1} \le \max_{1\le u\le t+1}\beta_{u-1}$ for all $s\le t$ and $C = \sum_{s=1}^t |c_s|$.

From~\eqref{equ:termB2_1} and~\eqref{equ:termB2_2},  
with probability at least $1-2\delta$,  
the following holds for all $t\in[T]$,
\begin{align}
&\sum_{s=1}^t 
\left\langle
-\nabla\ell_s(\theta^*),
\hat{\theta}_s - \theta^*
\right\rangle
\Indi_{\mathcal{A}_s} \notag \\
&\le
\frac14\left(
\lambda\alpha 
+
\sum_{s=1}^t
\left\langle \frac{X_s}{\sigma_s},
\hat\theta_s - \theta^*
\right\rangle^2
\right) 
 +\left(
8t^{\frac{1-\epsilon}{2(1+\epsilon)}}
\sqrt{2\kappa}
(\log 3T)^{\frac{1-\epsilon}{2(1+\epsilon)}}
\left(\log\frac{2T^2}{\delta}\right)^{\frac{\epsilon}{1+\epsilon}}
\right)^2  \notag \\                                         
&\quad+
\sqrt{\kappa}
\max_{u\in[t+1]} \beta_{u-1}.
\label{equ:termB2}
\end{align}

Combining~\eqref{equ:termB1} and~\eqref{equ:termB2} and applying the union bound,
we obtain that, with probability at least $1-3\delta$,
for all $t\geq 1$,
\begin{align*}
&\sum_{s=1}^t 
\left\langle 
\nabla \tilde{\ell}_s(\hat{\theta}_s)
- \nabla \ell_s(\hat{\theta}_s),
\hat{\theta}_s - \theta^*
\right\rangle 
\Indi_{\mathcal{A}_s} \\
&\le
65\sqrt{\alpha+\tfrac18}
\left(\log\frac{2T^2}{\delta}\right)  \tau_0
t^{\frac{1-\epsilon}{2(1+\epsilon)}}
\max_{u\in[t+1]} \beta_{u-1}
+ 12\sqrt{1+\tfrac{1}{8\alpha}}
\sqrt{\kappa}
t^{\frac{1-\epsilon}{2(1+\epsilon)}}
\max_{u\in[t+1]} \beta_{u-1}
\\
&\quad
+ \frac14\left(
\lambda\alpha
+ \sum_{s=1}^t 
\left\langle \frac{X_s}{\sigma_s},
\hat{\theta}_s - \theta^*
\right\rangle^2
\right)
+ \left(
8t^{\frac{1-\epsilon}{2(1+\epsilon)}}
\sqrt{2\kappa}
(\log 3T)^{\frac{1-\epsilon}{2(1+\epsilon)}}
\left(\log\frac{2T^2}{\delta}\right)^{\frac{\epsilon}{1+\epsilon}}
\right)^2 \\
&\qquad+ \sqrt{\kappa}
\max_{u\in[t+1]} \beta_{u-1},
\end{align*}
which completes the proof of Lemma~\ref{lem:gene_gap}.
\end{proof}

\subsection{Proof of Lemma~\ref{lem:conf_interval}}
\label{sec:conf_interval_proof}
\begin{proof}
The proof proceeds in four steps.
\paragraph{Step 1: Collecting the two high-probability bounds.}
To upper bound the right-hand side of Lemma~\ref{lem:key_decomp}, we consider an expression of a similar form and apply Lemma~\ref{lem:stability} and~\ref{lem:gene_gap}.
Applying a union bound, we obtain that, with probability at least $1-4\delta$, for all $t \in [T]$,
\begin{align*}
& 4\lambda S^2
+ \sum_{s=1}^{t}
\|\nabla \ell_s(\hat{\theta}_s)\|_{V_s^{-1}}^2
+ 2 \sum_{s=1}^{t}
\left\langle
\nabla \tilde{\ell}_s(\hat{\theta}_s) - \nabla \ell_s(\hat{\theta}_s),
\hat{\theta}_s - \theta^*
\right\rangle \Indi_{\mathcal{A}_s} 
+ \Bigl( \tfrac{1}{\alpha} - 1 \Bigr)
\sum_{s=1}^{t}
\left\|
\hat{\theta}_s - \theta^*
\right\|^2_{ \frac{X_s X_s^\top}{\sigma_s^{2}} } \\
&\le 
4\lambda S^2
+ 18\alpha 
\left[
t^{\frac{1 - \epsilon}{2(1 + \epsilon)}} 
\sqrt{2\kappa}(\log 3T)^{\frac{1 - \epsilon}{2(1 + \epsilon)}} 
\left( \log \frac{2T^2}{\delta} \right)^{\frac{\epsilon}{1 + \epsilon}} 
\right]^2  
+ \frac38 \sum_{s=1}^t
\left\| 
\theta^* - \hat{\theta}_s 
\right\|^2_{ \frac{X_s X_s^\top}{\sigma_s^2} }
+ 3\kappa \\
&\quad 
+ 130 \sqrt{ \alpha + \tfrac18 }
\Bigl(\log\tfrac{2T^2}{\delta}\Bigr)
\tau_0
t^{\frac{1-\epsilon}{2(1+\epsilon)}}
\max_{u\in[t+1]} \beta_{u-1}  
+ 24 \sqrt{1 + \tfrac{1}{8\alpha}}\sqrt{\kappa}
t^{\frac{1-\epsilon}{2(1+\epsilon)}}
\max_{u\in[t+1]} \beta_{u-1} \\
&\quad 
+ \frac12
\left(
\lambda \alpha 
+ \sum_{s=1}^t 
\left\langle 
\frac{X_s}{\sigma_s}, 
\hat{\theta}_s - \theta^*
\right\rangle^2
\right) 
+ 2
\left(
8 t^{\frac{1 - \epsilon}{2(1 + \epsilon)}}
\sqrt{2\kappa} (\log 3T)^{\frac{1 - \epsilon}{2(1 + \epsilon)}}
\left( 
\log \frac{2T^2}{\delta}
\right)^{\frac{\epsilon}{1 + \epsilon}}
\right)^2 \\
&\qquad 
+ 2\sqrt{\kappa}
\max_{u\in[t+1]} \beta_{u-1}
+ \Bigl( \tfrac{1}{\alpha} - 1 \Bigr)
\sum_{s=1}^{t} 
\left\| 
\hat{\theta}_s - \theta^*
\right\|^2_{ \frac{X_s X_s^\top}{\sigma_s^{2}} }.
\end{align*}
\paragraph{Step 2: Simplifying the coefficients.}
For $t\in[T]$, our parameter choices ensure that
\[
\sqrt{2}t^{\frac{1-\epsilon}{2(1+\epsilon)}}
(\log3T)^{\frac{1-\epsilon}{2(1+\epsilon)}}
\Bigl(\log\tfrac{2T^2}{\delta}\Bigr)^{\frac{\epsilon}{1+\epsilon}}
\ge 1.
\]
Hence, we multiply the terms of $3\kappa$ or
$2\sqrt{\kappa}\max_{u\in[t+1]}\beta_{u-1}$
by this factor to match the scale of the other terms.
Recall that
\[
(\log\tfrac{2T^2}{\delta})\tau_0
t^{\frac{1-\epsilon}{2(1+\epsilon)}}
=
t^{\frac{1-\epsilon}{2(1+\epsilon)}}
\sqrt{2\kappa}
(\log3T)^{\frac{1-\epsilon}{2(1+\epsilon)}}
(\log\tfrac{2T^2}{\delta})^{\frac{\epsilon}{1+\epsilon}},
\]
so that all terms on the right-hand side of the above form can be rewritten in terms of $(\log\tfrac{2T^2}{\delta})\tau_0
t^{\frac{1-\epsilon}{2(1+\epsilon)}}.$

Substituting $\alpha=8$ into the bound obtained in Step~1 and
collecting terms of the same form, we obtain that, with probability at least $1-4\delta$, for all $t \in [T]$,
\begin{align*}
& 4\lambda S^2
+ \sum_{s=1}^{t}
\|\nabla \ell_s(\hat{\theta}_s)\|_{V_s^{-1}}^2
+ 2 \sum_{s=1}^{t}
\left\langle
\nabla \tilde{\ell}_s(\hat{\theta}_s) - \nabla \ell_s(\hat{\theta}_s),
\hat{\theta}_s - \theta^*
\right\rangle \Indi_{\mathcal{A}_s} \notag\\
&\quad
+ \Bigl( \tfrac{1}{\alpha} - 1 \Bigr)
\sum_{s=1}^{t}
\left\|
\hat{\theta}_s - \theta^*
\right\|^2_{ \frac{X_s X_s^\top}{\sigma_s^{2}} } \\
&\le (4S^2+2)\lambda
+ (144+128+3)\Bigl((\log\tfrac{2T^2}{\delta})\tau_0
t^{\tfrac{1-\epsilon}{2(1+\epsilon)}}\Bigr)^2 \\
&\quad
+ \Bigl(130\sqrt{8+\tfrac{1}{8}}
      + 24\sqrt{1+\tfrac{1}{64}}
      + 2 \Bigr)
  (\log\tfrac{2T^2}{\delta})\tau_0
  t^{\tfrac{1-\epsilon}{2(1+\epsilon)}}
  \max_{u\in[t+1]}\beta_{u-1} \\
&\le (4S^2+2)\lambda
  + 275\Bigl((\log\tfrac{2T^2}{\delta})\tau_0
   t^{\tfrac{1-\epsilon}{2(1+\epsilon)}}\Bigr)^2 \\
&\quad + 397(\log\tfrac{2T^2}{\delta})\tau_0
   t^{\tfrac{1-\epsilon}{2(1+\epsilon)}}
   \max_{u\in[t+1]}\beta_{u-1}.
\end{align*}
Here, the term $\sum_{s=1}^t\|\hat{\theta}_s - \theta^*\|_{\frac{X_s X_s^{\top}}{\sigma_s}}$ cancels out because of the choice $\alpha=8$.

To match the standard confidence bound form
$\|\hat{\theta}_t - \theta^*\|_{V_t} \le \beta_t$,
it suffices to choose $\beta_t$ so that, for all $t \in [T]$,
\begin{align*}
    \beta_t^2\geq   (4S^2+2)\lambda
  + 275\Bigl((\log\tfrac{2T^2}{\delta})\tau_0
   t^{\tfrac{1-\epsilon}{2(1+\epsilon)}}\Bigr)^2+ 397(\log\tfrac{2T^2}{\delta})\tau_0
   t^{\tfrac{1-\epsilon}{2(1+\epsilon)}}
   \max_{u\in[t+1]}\beta_{u-1}.
\end{align*}
To simplify the notation, define
\[
b=397t^{\tfrac{1-\epsilon}{2(1+\epsilon)}},
c=(4S^2+2)\lambda
+275\Bigl((\log\tfrac{2T^2}{\delta})\tau_0
t^{\tfrac{1-\epsilon}{2(1+\epsilon)}}\Bigr)^2.
\]
Then the above condition can be written as
\[ 
b\max_{u\le t+1}\beta_{u-1}
+ c \leq \beta_t^2.
\]
\paragraph{Step 3: Choosing the confidence radius.}
By choosing the sequence $\{\beta_t\}_{t=1}^T$ to be non-decreasing, we have
$\max_{u\le t+1}\beta_{u-1} = \beta_t$. Therefore, it suffices to require
\[
b\beta_t + c \le \beta_t^2
\qquad\text{for all } t\in[T].
\]
The quadratic inequality holds whenever
\[
\beta_t
  \ge  
\frac{b}{2}
+\sqrt{
\Bigl(\frac{b}{2}\Bigr)^2
+ c}.
\]
Using the inequality $\sqrt{a+b}\le\sqrt{a}+\sqrt{b}$ for $a,b>0$, we obtain the simpler sufficient condition
\[
\beta_t
  \ge  
b
+ \sqrt{c}.
\]
Recalling the explicit form of $b$ and $c$, and slightly enlarging the numerical constant for simplicity, we define
\[
\beta_t
:= 409
(\log\tfrac{2T^2}{\delta})\tau_0
t^{\tfrac{1-\epsilon}{2(1+\epsilon)}}
+ \sqrt{(4S^2+2)\lambda}.
\]
With this choice, with probability at least $1-4\delta$,
the following inequality holds for all $t\in[T]$:
\begin{align}
    &\beta_t^2 \geq 4\lambda S^2
+ \sum_{s=1}^{t}
\|\nabla \ell_s(\hat{\theta}_s)\|_{V_s^{-1}}^2
+ 2 \sum_{s=1}^{t}
\left\langle
\nabla \tilde{\ell}_s(\hat{\theta}_s) - \nabla \ell_s(\hat{\theta}_s),
\hat{\theta}_s - \theta^*
\right\rangle \Indi_{\mathcal{A}_s}  \notag \\
&\quad + \Bigl( \tfrac{1}{\alpha} - 1 \Bigr)
\sum_{s=1}^{t}
\left\|
\hat{\theta}_s - \theta^*
\right\|^2_{ \frac{X_s X_s^\top}{\sigma_s^{2}}}.
\label{equ:sta_and_gene}
\end{align}
\paragraph{Step 4: Verifying the confidence radius via induction.}
Let $\mathcal{B}$ denote the event that the conditions in~\eqref{equ:sta_and_gene} hold for all
$t \ge 1$. By construction, $\mathbb{P}(\mathcal{B}) \ge 1 - 4\delta$.
We define the event
\begin{equation*}
\mathcal{C}
:= \bigcap_{t\ge 1}
\Bigl\{
\|\hat{\theta}_t-\theta^*\|_{V_{t-1}} \le \beta_{t-1}
\Bigr\}
=
\bigcap_{t\ge 1} \mathcal{A}_t.
\end{equation*}
We show that $\mathcal{B} \subseteq \mathcal{C}$ by mathematical induction,
which immediately implies
$\mathbb{P}(\mathcal{C}) \ge \mathbb{P}(\mathcal{B}) \ge 1 - 4\delta$.
By definition, $\mathcal{A}_1$ holds since
$\|\hat{\theta}_1 - \theta^\ast\|_{V_0}
\le \sqrt{4\lambda S^2}
\le \sqrt{\lambda(2+4S^2)}
= \beta_0 .$

Assume that at iteration $t \ge 1$, $\mathcal{A}_s$ holds for all $s \in [t]$.
We now show that $\mathcal{A}_{t+1}$ also holds.
By Lemma~\ref{lem:key_decomp}, we have
\begin{align*}
&\|\hat{\theta}_{t+1}-\theta^*\|_{V_t}^2 \\
&\le
4\lambda S^2
+ \sum_{s=1}^t \|\nabla \ell_s(\hat{\theta}_s)\|_{V_{s-1}^{-1}}^2
+ 2\sum_{s=1}^t
\bigl\langle
\nabla \tilde{\ell}_s(\hat{\theta}_s)
- \nabla \ell_s(\hat{\theta}_s),
\hat{\theta}_s - \theta^*
\bigr\rangle +\Bigl(\frac{1}{\alpha}-1\Bigr)
\sum_{s=1}^t
\|\hat{\theta}_s-\theta^*\|_{\frac{X_s X_s^\top} {\sigma_s^2}}^2
\\
&=
4\lambda S^2
+ \sum_{s=1}^t \|\nabla \ell_s(\hat{\theta}_s)\|_{V_{s-1}^{-1}}^2
+ 2\sum_{s=1}^t
\bigl\langle
\nabla \tilde{\ell}_s(\hat{\theta}_s)
- \nabla \ell_s(\hat{\theta}_s),
\hat{\theta}_s - \theta^*
\bigr\rangle \mathbf{1}_{\mathcal{A}_s}  \\
&\quad + \Bigl(\frac{1}{\alpha}-1\Bigr)
\sum_{s=1}^t
\|\hat{\theta}_s-\theta^*\|^2_{\frac{X_s X_s^\top}{\sigma_s^2}}
\\
&\le
\beta_t^2 .
\end{align*}
Here, the first inequality follows from Lemma~\ref{lem:key_decomp}, the first equality uses the
induction hypothesis that $\mathcal{A}_s$ holds for all $s \in [t]$, and the
second inequality follows from~\eqref{equ:sta_and_gene}. 
Consequently, $\mathcal{A}_{t+1}$ holds.
Since $t\ge 1$ was arbitrary, we conclude that $\mathcal{A}_t$ holds for all $t \ge 1$, which implies that
$\mathbb{P}(\mathcal{C}) \ge 1-4\delta$.

Moreover, we can verify that
\begin{align*}
\sqrt{
\frac{2\|X_t\|_{V_{t-1}^{-1}}^2 \beta_{t-1}}
{\sqrt{\alpha}\tau_0 t^{\frac{1-\epsilon}{2(1+\epsilon)}}}
}
=
\sqrt{
\frac{
409 \left(\log \frac{2T^2}{\delta}\right) \tau_0 t^{\frac{1-\epsilon}{2(1+\epsilon)}}
+ \sqrt{\lambda(2+4S^2)}
}
{\tau_0 t^{\frac{1-\epsilon}{2(1+\epsilon)}}}
}
\|X_t\|_{V_{t-1}^{-1}}
&\ge
\sqrt{409}\|X_t\|_{V_{t-1}^{-1}} \\
&\ge
2\sqrt{2}\|X_t\|_{V_{t-1}^{-1}}.
\end{align*} Hence, our choice of $\sigma_t$ satisfies the lower bound condition required in Lemma~\ref{lem:stability}.

Consequently, under the choice of $\sigma_t$ and $\tau_t$,
\[
\sigma_t
:=
\max\left\{
\nu_t,
\frac{1}{\sqrt{T}}
\sqrt{
\frac{
\|X_t\|_{V_{t-1}^{-1}}^2 \beta_{t-1}
}
{\tau_0 t^{\frac{1-\epsilon}{2(1+\epsilon)}}}
},
\sqrt{C}\kappa^{-1/4}\|X_t\|_{V_{t-1}}^{1/2}
\right\},
\qquad
\tau_t
=
\tau_0
\frac{\sqrt{1+w_t^2}}{w_t}
t^{\frac{1-\epsilon}{2(1+\epsilon)}},
\]
we obtain that for any $\delta \in (0,\frac{1}{4})$, with probability at least
$1-4\delta$, the following holds for all $t \ge 1$:
\[
\|\hat{\theta}_{t+1} - \theta_*\|_{V_t}
\le
409 \log \frac{2T^2}{\delta}
\tau_0 t^{\frac{1-\epsilon}{2(1+\epsilon)}}
+
\sqrt{\lambda(2+4S^2)}.
\]

\end{proof}
\section{Regret Analysis}
\subsection{Proof of Theorem~\ref{thm:regret_instance}}
\label{subsec:ins_reg}
\begin{proof}
We define $X_t^* = \argmax_{\mathbf{x}\in\mathcal{X}_t} \mathbf{x}^\top\theta^*.$
By Lemma~\ref{lem:conf_interval} and the fact that
$X_t^*, X_t \in \mathcal{X}_t$, we have, with probability at least
$1-4\delta$, for all $t\in[T]$,
\begin{equation}
\label{equ:Xt_star_bound}
\langle X_t^*, \theta^* - \hat{\theta}_t\rangle
\le
\|\theta^* - \hat{\theta}_t\|_{V_{t-1}}
\|X_t^*\|_{V_{t-1}^{-1}}
\le
\beta_{t-1}\|X_t^*\|_{V_{t-1}^{-1}},
\end{equation}
which can be rewritten as
\begin{equation}
\label{equ:X_t_star}
X_t^{*\top}\theta^*
\le
X_t^{*\top}\hat{\theta}_t
+ \beta_{t-1}\|X_t^*\|_{V_{t-1}^{-1}}.
\end{equation}
Similarly, by Lemma~\ref{lem:conf_interval},
with probability at least $1-4\delta$, for all $t\in[T]$,
\begin{equation}
\label{equ:Xt_bound}
\langle X_t, \theta^* - \hat{\theta}_t\rangle
\ge
-\beta_{t-1}\|X_t\|_{V_{t-1}^{-1}},
\end{equation}
which implies
\begin{equation}
\label{equ:X_t}
X_t^{\top}\theta^*
\ge
X_t^{\top}\hat{\theta}_t
- \beta_{t-1}\|X_t\|_{V_{t-1}^{-1}}.
\end{equation}

By applying the union bound to \eqref{equ:X_t_star} and \eqref{equ:X_t},
we obtain that, with probability at least $1-8\delta$, the following
holds for all $t\in[T]$:
\begin{align}
X_t^{*\top}\theta^* - X_t^\top\theta^*
&\le
\bigl(
  X_t^{*\top}\hat{\theta}_t
  + \beta_{t-1}\|X_t^*\|_{V_{t-1}^{-1}}
\bigr)
-
\bigl(
  X_t^{\top}\hat{\theta}_t
  - \beta_{t-1}\|X_t\|_{V_{t-1}^{-1}}
\bigr)
\notag \\
&=
X_t^{*\top}\hat{\theta}_t - X_t^\top\hat{\theta}_t
+ \beta_{t-1}\bigl(
    \|X_t^*\|_{V_{t-1}^{-1}}
    + \|X_t\|_{V_{t-1}^{-1}}
  \bigr).
\label{equ:reg_before}
\end{align}
On the other hand, by the arm selection rule~\eqref{eq:arm_selection}, 
\begin{equation*}
X_t^{*\top}\hat{\theta}_t
+ \beta_{t-1}\|X_t^*\|_{V_{t-1}^{-1}}
\le
X_t^\top\hat{\theta}_t
+ \beta_{t-1}\|X_t\|_{V_{t-1}^{-1}},
\end{equation*}
which implies
\begin{equation}
X_t^{*\top}\hat{\theta}_t - X_t^\top\hat{\theta}_t
\le
\beta_{t-1}\bigl(
  \|X_t\|_{V_{t-1}^{-1}}
  - \|X_t^*\|_{V_{t-1}^{-1}}
\bigr).
\label{equ:reg_arm_selection}
\end{equation}
Substituting~\eqref{equ:reg_arm_selection} into~\eqref{equ:reg_before} yields
\begin{align*}
X_t^{*\top}\theta^* - X_t^\top\theta^*
&\le 
\beta_{t-1}\bigl(
  \|X_t\|_{V_{t-1}^{-1}}
  - \|X_t^*\|_{V_{t-1}^{-1}}
\bigr)
+
\beta_{t-1}\bigl(
  \|X_t^*\|_{V_{t-1}^{-1}}
  + \|X_t\|_{V_{t-1}^{-1}}
\bigr)= 2\beta_{t-1}\|X_t\|_{V_{t-1}^{-1}}.
\end{align*}
Therefore, the regret satisfies
\begin{align*}
\mathrm{Reg}(T)
=
\sum_{t=1}^T
\bigl(
  X_t^{*\top}\theta^*
  - X_t^\top\theta^*
\bigr) \le
2\sum_{t=1}^T
\beta_{t-1}\|X_t\|_{V_{t-1}^{-1}} 
\le
2\beta_T
\sum_{t=1}^T
\|X_t\|_{V_{t-1}^{-1}} = 2\sqrt{\alpha}\beta_T
\sum_{t=1}^T \sigma_t w_t,
\end{align*}
where $\beta_t
=
409
\bigl(\log\tfrac{2T^2}{\delta}\bigr)\tau_0
t^{\tfrac{1-\epsilon}{2(1+\epsilon)}}
+ \sqrt{(4S^2+2)\lambda}.$

Next we bound the sum of bonuses $\sum_{t=1}^T \sigma_t w_t$
by splitting the time indices $[T]$ into three cases depending on which
candidate term attains the value of $\sigma_t$.
If the maximum is attained by multiple candidate terms,
we assign the index $t$ to an arbitrary one of them.
We define
\begin{align*}
\mathcal{J}_1 &:= \Bigl\{t\in[T]:
  \sigma_t \in \{\nu_t,\sigma_{\min}\}\Bigr\},\\
\mathcal{J}_2 &:= \Biggl\{t\in[T]:
  \sigma_t =
  \sqrt{\frac{2\beta_{t-1}}
             {\tau_0 \sqrt{\alpha}
              t^{\frac{1-\epsilon}{2(1+\epsilon)}}}}
   \|X_t\|_{V_{t-1}^{-1}}
  \Biggr\},\\
\mathcal{J}_3 &:= \Bigl\{t\in[T]:
  \sigma_t =
  \sqrt{C}\kappa^{-1/4}
    \|X_t\|_{V_{t-1}^{-1}}^{1/2}
  \Bigr\}.
\end{align*}

We first bound the contribution from the indices in $\mathcal{J}_1$.
Recalling that $\sigma_{\min} = 1/\sqrt{T}$, by the Cauchy--Schwarz
inequality and Lemma~\ref{lem:eliptical} we obtain
\begin{align}
  \sum_{t\in\mathcal{J}_1}\sigma_t w_t
  = \sum_{t\in\mathcal{J}_1}\max\{\nu_t, \sigma_{\min}\} w_t
  &\leq  \sum_{t=1}^T\max\{\nu_t, \sigma_{\min}\} w_t \notag \\
  &\leq
  \sqrt{\sum_{t=1}^{T}(\nu_t^2 + \sigma_{\min}^2)}
  \sqrt{\sum_{t=1}^T w_t^2} 
  \leq
  \sqrt{2\kappa}\sqrt{1+\sum_{t=1}^T \nu_t^2 },
\label{equ:J1_regret}
\end{align}
where in the last inequality we used $\sum_{t=1}^T \sigma_{\min}^2 = 1$
and Lemma~\ref{lem:eliptical} to bound $\sum_{t=1}^T w_t^2 \le 2\kappa$.

For the indices in $\mathcal{J}_2$, we use the fact that
Lemma~\ref{lem:conf_interval} holds, and hence the confidence
radius $\beta_t$ takes the form
\[
\beta_t
=
409
\biggl(\log\frac{2T^2}{\delta}\biggr)\tau_0
t^{\frac{1-\epsilon}{2(1+\epsilon)}}
+ \sqrt{\lambda(2+4S^2)}.
\]
Therefore,
\begin{align}
w_t^{-2}
= \frac{2\beta_{t-1}}
        {\sqrt{\alpha}\tau_0 t^{\frac{1-\epsilon}{2(1+\epsilon)}}}
&=\frac{2 \left(409
\biggl(\log\frac{2T^2}{\delta}\biggr)\tau_0
t^{\frac{1-\epsilon}{2(1+\epsilon)}}
+ \sqrt{\lambda(2+4S^2)}\right)}{\sqrt{\alpha}\tau_0 t^{\frac{1-\epsilon}{2(1+\epsilon)}}}\\
&\le
\frac{\sqrt{\lambda(2+4S^2)}}{\tau_0}
+ 409\log\frac{2T^2}{\delta},
\label{equ:w_t}
\end{align}
and letting
\[
c_0
:=
\frac{\sqrt{\lambda(2+4S^2)}}{\tau_0}
+ 409\log\frac{2T^2}{\delta},
\]
we obtain
\begin{equation}
\sum_{t\in\mathcal{J}_2}\sigma_t w_t
= \sum_{t\in\mathcal{J}_2}\frac{1}{w_t^2}\|X_t\|_{\tilde{V}^{-1}_{t-1}} w_t^2
\le
\frac{c_0L}{\sqrt{\alpha\lambda}} \sum_{t\in\mathcal{J}_2} w_t^2
\le
\frac{2c_0\kappa L}{\sqrt{\alpha\lambda}},
\label{equ:J2_regret}
\end{equation}
where the first inequality follows from
$\tilde V_{t-1}\succeq \lambda\alpha I_d$, which implies
$\|X_t\|_{\tilde V_{t-1}^{-1}}
\le L/\sqrt{\alpha\lambda}$,
and the second inequality follows from
Lemma~\ref{lem:eliptical}, which gives
$\sum_{t=1}^T w_t^2 \le 2\kappa$.

Finally, we bound the contribution from the indices in $\mathcal{J}_3$:
\begin{equation}
\sum_{t\in\mathcal{J}_3} \sigma_t w_t
= \sum_{t\in\mathcal{J}_3}
    \frac{\sqrt{\alpha}C}{\sqrt{\kappa}} w_t^2
\le
2\sqrt{\alpha\kappa} C,
\label{equ:J3_regret}
\end{equation}
where the inequality follows from Lemma~\ref{lem:eliptical}.
 
Choosing $\delta = 1/(8T)$ and
$\lambda = d$, we obtain $\kappa = \tilde{\mathcal{O}}(d)$ and $\beta_T
= \tilde{\mathcal{O}}\left(
  \sqrt{d}  
  T^{\frac{1-\epsilon}{2(1+\epsilon)}}
\right).$
Plugging these orders into the bounds
\eqref{equ:J1_regret}, \eqref{equ:J2_regret}, and
\eqref{equ:J3_regret}, and summing them up, with probability at least $1-\frac{1}{T}$, we have
\begin{align*}
\mathrm{Reg}(T)
&\le
2\sqrt{\alpha}\beta_T\left(
  \sqrt{2\kappa}\sqrt{1+\sum_{t=1}^T \nu_t^2}
  + \frac{2c_0\kappa L}{\sqrt{\alpha\lambda}}
  + \sqrt{\alpha\kappa} C
\right) \\
&=
\tilde{\mathcal{O}}\left(
  \sqrt{d}
  T^{\frac{1-\epsilon}{2(1+\epsilon)}}
  \left(
    \sqrt{d\sum_{t=1}^T \nu_t^2}
    + \sqrt{d}
  \right)
  + \sqrt{d} C
\right) \\
&=
\tilde{\mathcal{O}}\left(
  d T^{\frac{1-\epsilon}{2(1+\epsilon)}}
     \sqrt{\sum_{t=1}^T \nu_t^2}
  + d T^{\frac{1-\epsilon}{2(1+\epsilon)}}
  + d T^{\frac{1-\epsilon}{2(1+\epsilon)}} C
\right).
\end{align*}
\end{proof}

\subsection{Proof of Corollary~\ref{col:regret_c}}
\label{subsec:unknown_c_reg}
\begin{proof}
The argument follows the same structure as in the proof of
Theorem~\ref{thm:regret_instance}.  
The only difference is that
we replace $C$ with $\bar{C}$ in the definition of $\mathcal{J}_3$, i.e.,
\[
\mathcal{J}_3 := \{ t\in[T] : 
   \sigma_t = \sqrt{\bar{C}}\kappa^{-1/4}\|X_t\|_{V_{t-1}^{-1}}^{1/2}
   \}.
\]
With this definition of $\mathcal{J}_3$, we obtain the bound on the contribution of indices in $\mathcal J_3$:
\begin{equation*}
\sum_{t\in\mathcal{J}_3}\sigma_t w_t
\le 2\sqrt{\alpha\kappa}\bar{C}.
\end{equation*}

The bounds for $\mathcal{J}_1$ and $\mathcal{J}_2$ remain identical to
those in Theorem~\ref{thm:regret_instance}.  
Under the same choice of $\delta$ and $\lambda$ as in Theorem~\ref{thm:regret_instance}, with probability at least $1-\frac{1}{T}$, we have
\begin{align*}
\mathrm{Reg}(T)
&\le
2\sqrt{\alpha}\beta_T\left(
 \sqrt{2\kappa}\sqrt{\sum_{t=1}^T \nu_t^2 + 1}
 + \frac{2c_0L\kappa}{\sqrt{\alpha\lambda}}
 + 2\sqrt{\alpha\kappa}\bar{C}
\right) \\
&=
\tilde{\mathcal{O}}\left(
 dT^{\frac{1-\varepsilon}{2(1+\varepsilon)}}
   \sqrt{\sum_{t=1}^T \nu_t^2}
 + dT^{\frac{1-\varepsilon}{2(1+\varepsilon)}}
 + dT^{\frac{1-\varepsilon}{2(1+\varepsilon)}}\bar{C}
\right).
\end{align*}
\end{proof}

\subsection{Proof of Corollary~\ref{col:regret_nu}}
\label{subsec:unknown_nu_reg}
\begin{proof}
The argument follows the same structure as in the proof of
Theorem~\ref{thm:regret_instance}.
The only difference is that $\nu_t$ is replaced by $\nu$
in the definition of $\mathcal{J}_1$, namely,
\[
\mathcal{J}_1 := \{ t\in[T] : \sigma_t \in \{\nu,\sigma_{\min}\} \}.
\]
With this definition of $\mathcal{J}_1$, we obtain the bound on the contribution of indices in $\mathcal J_1$:
\begin{align*}
\sum_{t\in\mathcal{J}_1}\sigma_t w_t
&= \sum_{t\in\mathcal{J}_1} \max\{\nu,\sigma_{\min}\} w_t \\
&\le
\sqrt{\sum_{t\in\mathcal{J}_1}(\nu^2 + \sigma_{\min}^2)}
\sqrt{\sum_{t\in\mathcal{J}_1} w_t^2}
\le
\sqrt{2\kappa}\nu\sqrt{T+1}.
\end{align*}

The bounds for $\mathcal{J}_2$ and $\mathcal{J}_3$ remain identical to
those in Theorem~\ref{thm:regret_instance}.
Under the same choice of $\delta$ and $\lambda$ as in Theorem~\ref{thm:regret_instance}, with probability at least $1-\frac{1}{T}$, we have
\begin{align*}
\mathrm{Reg}(T)
&\le
2\sqrt{\alpha}\beta_T \sum_{t=1}^T \sigma_t w_t \\
&\le
4\beta_T\left(
  \sqrt{2\kappa}\nu\sqrt{T+1}
  + \frac{2c_0L\kappa}{\sqrt{\alpha\lambda}}
  + 2\sqrt{\alpha\kappa}C
\right) \\
&=
\tilde{\mathcal{O}}\left(
  dT^{\frac{1}{1+\varepsilon}}
  + dT^{\frac{1-\varepsilon}{2(1+\varepsilon)}}
  + dT^{\frac{1-\varepsilon}{2(1+\varepsilon)}}C
\right).
\end{align*}
\end{proof}
\section{Auxiliary Lemmas}
We collect several auxiliary lemmas adapted from
\citet{huang2023tackling} and \citet{wang2025heavy}.
For clarity, let $\mathcal{F}_t = \sigma(X_{1:t}, \eta_{1:t-1}, c_{1:t-1})$
denote the filtration associated with our process.
Compared to the filtrations used in the above works, ours additionally includes
the corruption sequence $c_{1:t-1}$.
This difference is immaterial, since the lemmas in
\citet{huang2023tackling} and \citet{wang2025heavy}
only require a conditional moment bound with respect to the underlying filtration.

\begin{lemma}[Partial result of Lemma~C.2 in \citet{huang2023tackling}]
\label{lem:huang_C.2}
Let $\{\mathcal{F}_t\}_{t=1}^\infty$ be a filtration and
$\{\eta_t \}_{t=1}^\infty$ be a sequence of random variables such that
for all $t\ge 1$,
\begin{equation*}
\mathbb{E}\left[\frac{\eta_t}{\sigma_t} \middle| \mathcal{F}_{t}\right]
= 0,
\qquad
\mathbb{E}\left[\left|\frac{\eta_t}{\sigma_t}\right|^{1+\epsilon}
 \middle| \mathcal{F}_{t}\right]
\le b^{1+\epsilon}.
\end{equation*}
For each $t\in\mathbb{N}$, let $X_t$ be an $\mathbb{R}^d$-valued random variable satisfying
$\|X_t\|_2 \leq L$. Define $U_0 = \lambda I_d$ and
$U_t = U_{t-1} + \frac{X_t X_t^\top}{\sigma_t^2}$ for $t\ge 1$, where $\lambda>0$.
Moreover, for $t\ge 1$ set
\[
\tau_t
=
\tau_0 \frac{\sqrt{1 + w_t^2}}{w_t} 
t^{\frac{1 - \epsilon}{2(1 + \epsilon)}},
\qquad
\tau_0
=
\frac{\sqrt{2\kappa}b 
      (\log 3T)^{\frac{1 - \epsilon}{2(1 + \epsilon)}}}
     {\bigl(\log \frac{2T^2}{\delta}\bigr)^{\frac{1}{1 + \epsilon}}}.
\]
Then, with probability at least $1 - 2\delta$, for all $t\ge 1$,
\begin{equation*}
\Biggl\|
\sum_{s=1}^t
f_{\tau_s}^{\prime}\left(\frac{\eta_s}{\sigma_s}\right)
\frac{X_s}{\sigma_s}
\Biggr\|_{U_t^{-1}}
\le
8 t^{\frac{1-\epsilon}{2(1+\epsilon)}} 
\sqrt{2\kappa} 
(\log 3T)^{\frac{1-\epsilon}{2(1+\epsilon)}} 
\left(\log\frac{2T^2}{\delta}\right)^{\frac{\epsilon}{1+\epsilon}}.
\end{equation*}
\end{lemma}

\begin{lemma}[Lemma~C.5 in \citet{huang2023tackling}]
\label{lem:huang_C.5}
Let $\{\mathcal{F}_t \}_{t=1}^\infty$ be a filtration and
$\{\eta_t \}_{t=1}^\infty$ be a sequence of random variables such that
for all $t\ge 1$,
\[
\mathbb{E}\left[\frac{\eta_t}{\sigma_t} \middle| \mathcal{F}_{t}\right]
= 0,
\qquad
\mathbb{E}\left[\left|\frac{\eta_t}{\sigma_t}\right|^{1+\epsilon}
 \middle| \mathcal{F}_{t}\right]
\le b^{1+\epsilon}.
\]
Moreover, for $t\ge 1$ set
\[
\tau_t
=
\tau_0 \frac{\sqrt{1 + w_t^2}}{w_t} 
t^{\frac{1 - \epsilon}{2(1 + \epsilon)}},
\qquad
\tau_0
=
\frac{\sqrt{2\kappa}b 
      (\log 3T)^{\frac{1 - \epsilon}{2(1 + \epsilon)}}}
     {\bigl(\log (2T^2/\delta)\bigr)^{\frac{1}{1 + \epsilon}}}.
\]
Then, with probability at least $1 - \delta$, for all $t\ge 1$,
\begin{align*}
\sum_{s=1}^t
{f_{\tau_s}^{\prime}}^2\left(\frac{\eta_s}{\sigma_s}\right)
\frac{w_s^2}{1 + w_s^2} 
\le
t^{\frac{1 - \epsilon}{1 + \epsilon}}
\left(
  \sqrt{
    \tau_0^{1 - \epsilon}
    (\sqrt{2 \kappa}b)^{1 + \epsilon}
    (\log 3t)^{\frac{1 - \epsilon}{2}}
  }
  +
  \tau_0 \sqrt{2 \log \frac{2t^2}{\delta}}
\right)^2.
\end{align*}
\end{lemma}

\begin{lemma}[Eq.~(C.12) in \citet{huang2023tackling}]
\label{lem:c_12}
Let $\{\mathcal F_t\}_{t=1}^{\infty}$ be a filtration and
$\{\eta_t \}_{t=1}^\infty$ be a sequence of random variables such that
for all $t\ge 1$,
\[
\mathbb{E}\left[\frac{\eta_t}{\sigma_t} \middle| \mathcal{F}_{t}\right]
= 0,
\qquad
\mathbb{E}\left[\left|\frac{\eta_t}{\sigma_t}\right|^{1+\epsilon}
 \middle| \mathcal{F}_{t}\right]
\le b^{1+\epsilon}.
\]
Moreover, for $t\ge 1$ set
\[
\tau_t
=
\tau_0 \frac{\sqrt{1 + w_t^2}}{w_t} 
t^{\frac{1 - \epsilon}{2(1 + \epsilon)}},
\qquad
\tau_0
=
\frac{\sqrt{2\kappa}b 
      (\log 3T)^{\frac{1 - \epsilon}{2(1 + \epsilon)}}}
     {\bigl(\log (2T^2/\delta)\bigr)^{\frac{1}{1 + \epsilon}}}.
\]
Then, with probability at least $1 - \delta$, for all $t\ge 1$,
\begin{equation*}
    \sum_{s=1}^{t} \Indi\left\{\left|\frac{\eta_s}{\sigma_s}\right|>\frac{\tau_s}{2}\right\}\leq \frac{23}{3}\log\frac{2T^2}{\delta}.
\end{equation*}
\end{lemma}

\begin{lemma}[Elliptical potential lemma, Lemma~7 in \citet{wang2025heavy}]
\label{lem:eliptical}
For each $t\in\mathbb{N}$, let $X_t$ be an $\mathbb{R}^d$-valued random variable satisfying
$\|X_t\|_2 \leq L$.
Let $U_0 = \lambda I_d$ and
$U_t = U_{t-1} + X_t X_t^\top$ for $t\ge 1$, where $\lambda>0$.
Assume that for all $t\ge 1$, $\|U_{t-1}^{-1/2} X_t\|_2^2 \leq c_{\max}.$
Then,
\begin{equation*}
\sum_{t=1}^T \|U_{t-1}^{-1/2} X_t\|_2^2
\leq
2 \max\{1, c_{\max}\}  d
\log\left(1 + \frac{L^2 T}{\lambda d} \right).
\end{equation*}
\end{lemma}

\end{document}